\documentclass[twoside]{article}
\usepackage{palatino,epsfig,latexsym,natbib}
\usepackage[margin=1in]{geometry}
\usepackage{algorithm}
\usepackage[noend]{algpseudocode}
\usepackage{amsmath}
\usepackage{tikz}
\usepackage{pgfplots}
\usepackage{amssymb}
\usepackage{thmtools}
\usepackage{thm-restate}
\usepackage{xcomment}
\usepackage{xspace}
\usepackage[utf8]{inputenc}

\newtheorem{theorem}{Theorem}

\newtheorem{lemma}[theorem]{Lemma}
\newtheorem{corollary}[theorem]{Corollary}
\newenvironment{proof}{\paragraph{Proof:}}{\hfill$\square$}
\DeclareMathOperator{\Imp}{Imp}
\newcommand{\CHANGED}[1]{{\color{black} #1}\xspace}
\newcommand{\NEWCHANGED}[1]{{\color{black} #1}\xspace}

\title{Simple Hyper-heuristics Control the Neighbourhood Size of Randomised Local Search Optimally for LeadingOnes\thanks{An extended abstract of this manuscript has appeared at the 2017 Genetic and Evolutionary Computation Conference (GECCO 2017)~\citep{LissovoiEtAl2017}.}}
\date{}
\author{\textbf{Andrei Lissovoi} \hfill a.lissovoi@sheffield.ac.uk\\ Department of Computer Science, University of Sheffield, UK \and \textbf{Pietro S. Oliveto} \hfill p.oliveto@sheffield.ac.uk\\ Department of Computer Science, University of Sheffield, UK \and \textbf{John Alasdair Warwicker} \hfill j.warwicker@sheffield.ac.uk \\ Department of Computer Science, University of Sheffield, UK}

\begin{document}



\maketitle

\begin{abstract}
\NEWCHANGED{Selection hyper-heuristics (HHs) are randomised search methodologies which choose and execute heuristics during the optimisation process from a set of low-level heuristics. A machine learning mechanism is generally used to decide which low-level heuristic should be applied in each decision step. In this paper we analyse whether sophisticated learning mechanisms are always necessary for HHs to perform well. To this end we consider the most simple HHs from the literature and rigorously analyse their performance for the \textsc{LeadingOnes} benchmark function. Our analysis shows that the standard Simple Random, Permutation, Greedy and Random Gradient HHs show no signs of learning. While the former HHs do not attempt to learn from the past performance of low-level heuristics, the idea behind the Random Gradient HH is to continue to exploit the currently selected heuristic as long as it is successful. Hence, it is embedded with a reinforcement learning mechanism with the shortest possible memory. However, the probability that a promising heuristic is successful in the next step is relatively low when perturbing a reasonable solution to a combinatorial optimisation problem. We generalise the `simple' Random Gradient HH so success can be measured over a fixed period of time $\tau$, instead of a single iteration. For \textsc{LeadingOnes} we prove that the {\it Generalised Random Gradient (GRG)} HH can learn to adapt the neighbourhood size of Randomised Local Search to optimality during the run. As a result, we prove it has the best possible performance achievable with the low-level heuristics (Randomised Local Search with different neighbourhood sizes), up to lower order terms. We also prove that the performance of the HH improves as the number of low-level local search heuristics to choose from increases. In particular, with access to $k$ low-level local search heuristics, it outperforms the best-possible algorithm using any subset of the $k$ heuristics. Finally, we show that the advantages of GRG over Randomised Local Search and Evolutionary Algorithms using standard bit mutation increase if the anytime performance is considered (i.e., the performance gap is larger if approximate solutions are sought rather than exact ones). Experimental analyses confirm these results for different problem sizes (up to $n=10^8$) and shed some light on the best choices for the parameter $\tau$ in various situations.}\end{abstract}

\section{Introduction}
Many successful applications of randomised search heuristics to real-world optimisation problems have been reported. Despite these successes, it is still difficult to decide which particular search heuristic is a good choice for the problem at hand and what parameter settings should be used. In particular, while it is well understood that each heuristic will be efficient on some classes of problems and inefficient on others \citep{WolpertMacready1997}, very little guidance is available explaining how to choose an algorithm for a given problem. The high level idea behind the field of hyper-heuristics (HHs) is to overcome this difficulty by evolving the search heuristic for the problem rather than choosing one in advance (or applying several arbitrary ones until a satisfactory solution is found). The overall goal is to automate the design and the tuning of the algorithm and parameters for the optimisation problem, hence \CHANGED{to achieve} a more generally applicable system. 

\CHANGED{HHs} are usually classified into two categories: \CHANGED{\textit{generation} HHs and \textit{selection} HHs \citep{BurkeEtAl2013}. The former typically work offline and generate new heuristics from components of existing heuristics, while the latter typically work online and repeatedly select from a set of low-level heuristics which one to apply in the next state of the optimisation process}. The low-level heuristics can be further classified as either \textit{construction heuristics}, which build a solution incrementally, or \textit{perturbation heuristics}, which start with a complete solution and try to iteratively improve \CHANGED{it}. 

\CHANGED{In this paper, selection-perturbation HHs will be considered. Selection HHs belong to the wide field of {\it algorithm selection from algorithm portfolios}~\citep{BezerraEtAl2018}. A noteworthy example of such a system is the SATZilla approach for the SAT problem~\citep{XuEtAl2008}. We refer to \citep{Kotthoff2016} for an overview of approaches to algorithm selection for combinatorial optimisation problems.
It should be noted that HHs may also be used for other purposes. For example, a HH may be used as a parameter control method if it is asked to select at each state between different heuristics that only differ by one parameter value, 
rather than between different general purpose algorithms. See \citep{DoerrDoerrBookChapter2018} for a comprehensive review of parameter control techniques. In general though,} while the aim of parameter control mechanisms is to essentially choose between different parameter values for a given algorithm, selection HHs \CHANGED{normally} work at a higher level, since their aim is to learn which low-level heuristics (i.e., algorithms) work best at different states of the optimisation process \CHANGED{i.e., very different algorithms may perform better in different areas of the search space.}

\CHANGED{Selection HHs consist of two separate components: (1) a {\it heuristic selection method}, often referred to as the {\it learning mechanism}, to decide which heuristic to be applied in the next step of the optimisation process, and (2) a {\it move acceptance operator} to decide whether the newly produced search points should be accepted. 
 
Move acceptance operators are classified into {\it deterministic} ones, which make the same decisions independent of the state of the optimisation process, and {\it non-deterministic} ones, which might make different decisions for the same solutions at different times. 
 
 The majority of heuristic selection methods in the literature apply machine learning techniques that generate scores for each heuristic based on their past performance. A commonly used method for the purpose is reinforcement learning~\citep{CowlingEtAl2000,Nareyek2004,BurkeEtAlJournal2003}. 
 However, also considerably simpler heuristic selection methods have been successfully applied  in the literature. These include, among others: {\it Simple Random} heuristic selection which just chooses heuristics uniformly at random; {\it Random Permutation} which generates a random ordering of the low-level heuristics and applies them in that order; {\it Greedy} heuristic selection which applies all the heuristics to the current search point and returns the best found solution; {\it Random Gradient} which chooses a random heuristic and continues to apply it as long as it is successful, and {\it Random Permutation Gradient} which is a combination of the Permutation and Random Gradient methodologies~\citep{CowlingEtAl2000,CowlingEtAl2002,BerberogluUyar2011}. 
 
Apart from the Greedy one, an obvious advantage of these simple methods is that they execute very quickly. On the other hand, it is unclear how well they perform in terms of how effective they are \CHANGED{at} identifying (near) optimal solutions in a short time. While the Simple Random, Random Permutation and Greedy mechanisms do not attempt to learn anything from the past performance of the applied heuristics, Random Gradient and Permutation Random Gradient can still be considered as intelligent selection mechanisms because they embed a reinforcement learning mechanism, albeit with the shortest memory length possible. Experimental work has suggested that such mechanisms may be useful for highly rugged search landscapes~\citep{BurkeEtAl2013}.  In this paper we will rigorously evaluate the performance of simple heuristic selection methods for a unimodal benchmark problem and focus most of our attention on the Random Gradient method that uses 'minimal' intelligence.}

Numerous successes of selection HHs for NP-hard optimisation problems have been reported, including scheduling \citep{CowlingEtAl2000,CowlingEtAl2002,GibbsEtAl2010}, timetabling \citep{OzcanEtAl2012}, vehicle routing \citep{AstaOzcan2014}, cutting and packing \citep{LopezEtAl2014}. 
\CHANGED{However, their theoretical understanding is very limited.} Some insights into the behaviour of selection \CHANGED{HHs} have been achieved via landscape analyses \citep{MadenEtAl2009, OchoaEtAl2009GECCO, OchoaEtAl2009CEC}. \CHANGED{Concerning their performance, the most sophisticated HH that has been analysed is the one considered by \cite{DoerrEtAl2016}, where the different neighbourhood sizes $k$ of the Randomised Local Search (RLS$_k$) algorithm were implicitly used as low-level heuristics and a reinforcement learning mechanism was applied as the heuristic selection method. The authors proved that this HH can track the best local search neighbourhood size (i.e., the best fitness-dependent number of distinct bits to be flipped) for \textsc{OneMax}, while they showed experimentally that it outperforms RLS$_1$ and the (1+1) Evolutionary Algorithm ((1+1) EA) for \textsc{LeadingOnes}. The few other available runtime analyses consider the less sophisticated heuristic selection methodologies which are the focus of this paper.}

\cite{LehreOzcan2013} presented the first runtime analysis of selection-perturbation \CHANGED{HHs}. They considered \CHANGED{the} Simple Random \CHANGED{HH} \citep{CowlingEtAl2000} that at each step randomly chooses between \CHANGED{a 1-bit flip and a 2-bit flip operator} as low-level heuristics, and presented an example benchmark function class, called \textsc{GapPath}, where it is necessary to use more than one \CHANGED{of the} low-level heuristic\CHANGED{s} to optimise the problem. Similar example functions have \CHANGED{also} been constructed by \cite{HeEtAl2012}.

\begin{sloppypar}
A comparative time-complexity analysis of selection \CHANGED{HHs} has been presented by \cite{AlanaziLehre2014}. They considered several \CHANGED{of the common simple} selection mechanisms, namely Simple Random, Permutation, Random Gradient and Greedy \citep{CowlingEtAl2000,CowlingEtAl2002} and analysed their performance \CHANGED{for} the standard \textsc{LeadingOnes} benchmark function when using a low-level heuristic set consisting of a 1-bit flip and a 2-bit flip operator (i.e., the same set previously considered by \cite{LehreOzcan2013}). Their runtime \CHANGED{analyses} show that the four simple selection mechanisms have the same asymptotic expected runtime, while an experimental evaluation \CHANGED{indicates} that the runtimes are indeed equivalent already for small problem dimensions. Recently, additive Reinforcement Learning selection was also shown \CHANGED{(under mild assumptions)} to often have asymptotically equivalent performance to Simple Random selection, including for the same problem setting (i.e., \textsc{LeadingOnes} selecting between 1-bit flip and 2-bit flip operators) \citep{AlanaziLehre2016}. In particular, the results indicate that selection mechanisms such as additive Reinforcement Learning and Random Gradient do not learn to exploit the more successful low-level heuristics and end up having the same performance as Simple Random selection. 
\end{sloppypar}

The main idea behind \CHANGED{Random Gradient} 
is to continue to exploit the currently selected heuristic as long as it is successful. Unlike construction heuristics, where iterating a greedy move on a currently successful heuristic may work for several consecutive construction steps, the probability that a promising heuristic is successful in the next step is relatively low when perturbing a reasonable solution to a combinatorial optimisation problem.

\CHANGED{To this end}, in this paper we \CHANGED{propose to} generalise the 
Random Gradient selection-perturbation 
mechanism \CHANGED{such that} success can be measured over \CHANGED{a} fixed period of time $\tau$\CHANGED{, which we call the learning period}, \CHANGED{instead of doing so after each} iteration. \CHANGED{We refer to the generalised HH as Generalised Random Gradient (GRG).} We use the \textsc{LeadingOnes} benchmark function \CHANGED{and the 1-bit flip and a 2-bit flip operators as low-level heuristics} to show that \CHANGED{with this simple modification Random Gradient} can be surprisingly fast, \CHANGED{even if no sophisticated machine learning mechanism is used to select from the set of low-level heuristics.} We first \CHANGED{derive the} exact leading constants in the expected runtimes of the simple mechanisms \CHANGED{for} \textsc{LeadingOnes}, \CHANGED{thus proving} that they all have expected runtime $\frac{\ln(3)}{2}n^2+o(n^2)\CHANGED{\phantom{x}\approx0.54931n^2+o(n^2)}$, confirming what was implied by the experimental analysis of \cite{AlanaziLehre2014}. This result indicates that all the simple mechanisms essentially choose operators at random in each iteration. \CHANGED{Thus they have worse performance} than the single operator that always flips one bit (i.e., \CHANGED{RLS$_1$}), which takes $\frac{1}{2}n^2$ expected iterations to optimise \textsc{LeadingOnes}. We then \CHANGED{provide upper bounds on the expected runtime of \CHANGED{GRG}} for the same function. We rigorously prove that the generalised \CHANGED{HH} has a better expected running time than RLS$_1$ for appropriately chosen values for the parameter $\tau$. \CHANGED{Furthermore, we prove that GRG can achieve an expected runtime of $\frac{1+\ln(2)}{4}n^2+o(n^2)\approx0.42329n^2\CHANGED{+o(n^2)}$ for \textsc{LeadingOnes}, when $\tau$ satisfies both $\tau=\omega(n)$ and $\tau\leq\left(\frac{1}{2}-\varepsilon\right)n\ln(n)$, for some constant $\CHANGED{0<\varepsilon<\frac{1}{2}}$. This is} the best possible \CHANGED{expected} runtime \NEWCHANGED{for an unbiased (1+1) black box algorithm using only the same two mutation operators}. \CHANGED{We refer to such performance as {\it optimal} for the HH as no better expected runtime, up to lower order terms, may be achieved with any combination of the available low-level heuristics.}

\CHANGED{Afterwards we turn our attention to low-level heuristic sets that contain arbitrarily many operators (i.e., $\{1,\dots, k\}$-bit flips, $k=\Theta(1)$)} 
as \CHANGED{commonly used} in practical applications. We first show that including more operators is detrimental to the performance of the simple mechanisms \CHANGED{for \textsc{LeadingOnes}. Then}, we prove that the performance of GRG improves as the number of available operators to choose from increases. In particular, when choosing from $k$ operators as low-level heuristics, GRG is in expectation faster than the best possible performance achievable \CHANGED{by} any algorithm using \CHANGED{any combination of} $m<k$ \CHANGED{of the} operators, for \CHANGED{any} $k=\Theta(1)$.

We conclude the paper with an experimental analysis of the \CHANGED{HH} for increasing problem sizes, up to $n=10^8$. The experiments confirm that \CHANGED{GRG} outperforms its low-level heuristics. For two operators, it is shown how proper choices for the parameter $\tau$ can lead to the near optimal performance already for these problem sizes. 
When the \CHANGED{HH} 
\CHANGED{is allowed} to choose between more than two operators, the experiments show that, \CHANGED{already for the considered problem sizes}, the performance of GRG improves with more operators for appropriate choices of the learning period $\tau$. 

The paper is structured as follows. In Section \ref{sec:preliminaries}, we formally introduce the \CHANGED{HH} framework together with the simple selection \CHANGED{HHs} and \CHANGED{GRG}. In Section \ref{sec:leadingones} we analyse the simple and generalised \CHANGED{HHs} \CHANGED{using two low-level heuristics} for the \textsc{LeadingOnes} benchmark function. In Section \ref{sec:MoreThanTwo} we present the results for the \CHANGED{HHs} that choose between 
\CHANGED{an arbitrary number of} low-level heuristics. Section \ref{sec:experiments} presents the experimental analysis. In the conclusion we present a discussion and some avenues for future work.

Compared to its conference version \citep{LissovoiEtAl2017} this manuscript has been considerably extended. The results of Section~\ref{sec:leadingones} have been generalised to hold for any values of $\tau$ less than $\left(\frac{1}{2}-\varepsilon\right)n\ln n$, for some constant $\CHANGED{0<\varepsilon<\frac{1}{2}}$. In general, the results have been considerably strengthened. In particular, compared to the extended abstract, in the present manuscript we prove that GRG optimises \textsc{LeadingOnes} in the best possible expected runtime achievable, up to lower order terms. Section~\ref{sec:MoreThanTwo} is a completely new addition, while a more comprehensive set of experiments is included in Section~\ref{sec:experiments}.

\section{Preliminaries}\label{sec:preliminaries}
\begin{algorithm}[t]
\caption{Simple Selection Hyper-heuristic \citep{CowlingEtAl2000,CowlingEtAl2002,AlanaziLehre2014}}
\begin{algorithmic}[1]
\State Choose $s \in S$ uniformly at random
\While{stopping conditions not satisfied}
\State Choose $h \in H$ according to the \textit{learning mechanism}
\State $s^\prime \gets h(s)$
\If{$f(s^\prime) > f(s)$}
\State $s \gets s^\prime$
\EndIf
\EndWhile
\end{algorithmic}
\label{SimpleHH}
\end{algorithm}
\subsection{The Hyper-heuristic Framework and the Heuristic Selection Methods}
Let $S$ be a finite search space, $H$ a set of low-level heuristics and $f:S\rightarrow\mathbb{R}^+$ a \CHANGED{fitness} function. Algorithm \ref{SimpleHH} shows the pseudocode representation for a simple selection \CHANGED{HH} as used in previous experimental and theoretical work \citep{CowlingEtAl2000,CowlingEtAl2002,AlanaziLehre2014}.
Throughout the paper, \CHANGED{the term} runtime \CHANGED{refers to the number of fitness evaluations used by Algorithm~\ref{SimpleHH} until it finds the optimum. Note that the strict inequality in line 5 of Algorithm~\ref{SimpleHH} is consistent with previous HH literature \citep{AlanaziLehre2014}, and using `$\geq$` instead would make no difference for the benchmark problem we consider in this work (i.e., \textsc{LeadingOnes}).}

\CHANGED{Different HHs are obtained from the framework described in Algorithm \ref{SimpleHH} according to which heuristic selection method is applied in Step 3. While sophisticated machine learning mechanisms are usually used,} 
the following \textit{learning mechanisms} have also been commonly \CHANGED{considered} in the literature to solve combinatorial optimisation problems \citep{CowlingEtAl2000, CowlingEtAl2002}:
\begin{itemize}
\item \textbf{Simple Random}, which selects a low-level heuristic \CHANGED{$h\in H$} independently with probability $p_h$ in each iteration. \CHANGED{Each heuristic is usually selected uniformly at random i.e., $p_h=\frac{1}{|H|}$};
\item \textbf{Permutation}, which generates a random ordering of \CHANGED{the} low-level heuristics \CHANGED{in $H$} and returns them in that sequence when called by the \CHANGED{HH};
\item \textbf{Greedy}, which applies all \CHANGED{the} low-level heuristics \CHANGED{in $H$} in parallel and returns the best found solution;
\item \textbf{Random Gradient}, which randomly selects a low-level heuristic \CHANGED{$h\in H$}, and keeps using it as long as it obtains improvements.
\end{itemize}

\cite{AlanaziLehre2014} derived upper and lower bounds on the expected runtime \CHANGED{of Algorithm~\ref{SimpleHH} for the \textsc{LeadingOnes} benchmark function when using these four simple heuristic selection mechanisms equipped with $H=\{\textsc{1BitFlip},\textsc{2BitFlip}\}$. The first low-level heuristic flips one bit chosen uniformly at random while the second one chooses two bits to flip with replacement (i.e., the same bit may flip twice).}

\CHANGED{The runtime analysis performed by \cite{AlanaziLehre2014} only provided bounds of the same asymptotic order on the expected runtime for all four mechanisms. An experimental analysis they carried out suggested, however, that}
all mechanisms have the same performance as just choosing the low-level heuristics at random \CHANGED{for \textsc{LeadingOnes}}. We will prove this by deriving the exact expected runtimes of the simple mechanisms in Section \ref{sec:leadingones}. 
\CHANGED{Differently from the other three mechanisms, the idea behind Random Gradient is to try to learn from the past performance of the heuristics.
However,}
by making a heuristic selection decision in every iteration, 
the mechanism does not have enough time to \CHANGED{appreciate whether the selected heuristic is a good choice for} 
the current optimisation state.

In this paper we generalise the Random Gradient mechanism to allow a longer \CHANGED{period of time to appreciate} whether a low-level heuristic is successful or not, 
\CHANGED{before deciding whether to select a different low-level heuristic.}
\CHANGED{Our aim is} to maintain the intrinsic ideas of the simple learning mechanism while generalising it sufficiently to allow for learning to take place. \CHANGED{The pseudocode of the  resulting HH is given in Algorithm~\ref{alg:GRG}, while the proposed learning mechanism works as follows:} 

\textbf{Generalised Random Gradient (GRG)}: A low-level heuristic is chosen uniformly at random (\textit{Decision Stage})
and run for a \CHANGED{learning} period of fixed time $\tau$. If an improvement is found before the end of the period, then a new period of time $\tau$ is immediately initialised (i.e., $c_t$ in Algorithm~\ref{alg:GRG} is set to 0 immediately) (\textit{Exploitation Stage}).
If the chosen operator fails to provide an improvement in $\tau$ iterations, a new operator is chosen at random.

\begin{algorithm}[t]
\caption{Generalised Random Gradient Hyper-heuristic}
\begin{algorithmic}[1]
\State Choose $x \in S$ uniformly at random
\While{stopping conditions not satisfied}
\State Choose $h \in H$ uniformly at random
\State $c_t \gets 0$
\While{$c_t < \tau$}
\State $c_t \gets c_t+1$; $x^\prime \gets h(x)$
\If{$f(x^\prime) > f(x)$}
\State $c_t \gets 0$; $x \gets x^\prime$
\EndIf
\EndWhile
\EndWhile
\end{algorithmic}
\label{alg:GRG}
\end{algorithm}

\CHANGED{Our aim is to prove that the simple GRG HH runs in the best possible expected runtime achievable for \textsc{LeadingOnes} using mutation operators with different neighbourhood sizes (i.e., $\{1\textsc{BitFlip},\dots,k\textsc{BitFlip}\}$ mutation operators) as low-level heuristics. We will first derive the best possible expected runtime achievable by applying the $k$ operators in any order, and then prove that GRG matches it up to lower order terms.

Throughout the paper, when discussing the $m\textsc{BitFlip}$ operator, we refer to the mutation operator which flips $m=\Theta(1)$ bits in the bit-string with replacement; that is, it is possible to flip and re-flip the same bit within the same mutation step. Since this has been used in previous literature on the topic \citep{LehreOzcan2013,AlanaziLehre2014,AlanaziLehre2016}, we naturally continue with this choice. We will also prove that the presented results \NEWCHANGED{hold also} if operators that flip $m$ bits without replacement are used (i.e., operators that select a new bit-string with Hamming distance $m$ from the original bit-string). In particular, we will show that any performance differences are limited to lower order terms in the expected runtimes. The latter operators are well known RLS algorithms with neighbourhood size $m$ (i.e., \textsc{RLS}$_m$). 
}

\subsection{The Pseudo-Boolean Benchmark Function \textsc{LeadingOnes}}
The \textsc{LeadingOnes} (\textsc{LO}) \CHANGED{benchmark} function counts the number of consecutive one-bits in a bit string before the first zero-bit:
$$\textsc{LeadingOnes}(x):=\sum_{i=1}^n\prod_{j=1}^i x_j.$$
The unrestricted black box complexity\footnote{\CHANGED{That is, the expected number of fitness function evaluations performed by the best-possible black-box algorithm until the optimum is found, i.e., there are no restrictions on which queries to the black box (i.e., which search points to be evaluated) the algorithm is allowed to make.}} of \textsc{LeadingOnes} is $O(n \log \log n)$ \citep{AfshaniEtAl2013} and 
there exist randomised search heuristics with expected runtimes of $o(n^2)$. 
Recently some Estimation of Distribution Algorithms (EDAs) have been presented which have surprisingly good performance for the problem. 
\cite{DoerrKrejca2018} introduced a modified compact Genetic Algorithm (cGA) called sig-cGA that, rather than updating the frequency vector \CHANGED{in} every generation, does so only once it notices a significance in its history of samples \NEWCHANGED{(i.e., once a significant number of 1s or 0s are sampled in a certain bit position, as opposed to a uniform binomial distribution of samples)}.
They proved that the algorithm optimises \textsc{LeadingOnes}, \textsc{OneMax} and \textsc{BinVal} in $O(n \log n)$ expected function evaluations. Two other algorithms have been proven to optimise \textsc{LeadingOnes} in the same asymptotic expected time: a stable compact Genetic Algorithm (scGA) that biases updates that favour frequencies that move towards $\frac{1}{2}$~\citep{FriedrichEtAl2016} and a Convex Search Algorithm (CSA) using binary uniform convex hull recombination (for sufficiently large populations and with an appropriate restart strategy)~\citep{MoraglioSudholt2017}. While the latter two algorithms have exceptional performance for \textsc{LeadingOnes}, their runtime is very poor for \textsc{OneMax}, respectively providing runtimes at least exponential and super-polynomial in the problem size with high probability~\citep{DoerrKrejca2018}.

The unbiased black box complexity of \textsc{LeadingOnes} is $\Theta(n^2)$~\citep{LehreWitt2012}. If biased mutation operators are allowed but truncation selection is imposed then no asymptotic improvement may be achieved over the unbiased black box complexity. Indeed, \cite{DoerrLengler2018} recently proved that the best possible asymptotic performance of any (1+1) elitist black-box algorithm for \textsc{LeadingOnes} is $\Omega(n^2)$. This bound is matched by the performance of simple well studied heuristics. \CHANGED{RLS} has an expected runtime of $0.5n^2$ fitness function evaluations \citep{BuzdalovBuzdalova2015} and it is well known that the standard (1+1) EA (with mutation rate $\frac{1}{n}$) has an expected runtime of $\frac{e-1}{2}n^2-o(n^2)\approx0.85914n^2\CHANGED{-o(n^2)}$ \citep{BoettcherEtAl2010}. \cite{BoettcherEtAl2010} also showed that the best static mutation rate for the (1+1) EA is $\approx\frac{1.5936}{n}$, which improves the expected runtime to $\frac{e}{4}n^2\pm O(n)\approx0.77201n^2\CHANGED{\pm O(n)}$. Furthermore, they showed that the (1+1) EA with an appropriately chosen dynamic mutation rate (i.e., $\frac{1}{\textsc{LO}(x)+1}$) can outperform any static choice, giving an expected runtime of \CHANGED{approximately} $0.68n^2\pm o(n^2)$. \CHANGED{Note that the (1+1) EA is slowed down by the approximately $37\%$ of the iterations in which no bits are flipped \citep{JansenZarges2011}. The best possible (1+1) EA variant that uses standard bit mutation (discounting iterations in which no bits flip) with the best possible mutation rate at each step has an expected runtime of approximately $0.404n^2\pm o(n^2)$ \citep{DoerrDoerrLengler2019Arxiv}. This is still worse than the best expected performance achievable by an unbiased (1+1) black box algorithm, approximately $0.388n^2\pm o(n^2)$ \citep{DoerrWagner2018,Doerr2018Arxiv}. In this paper we will prove that GRG equipped with a sufficient number of low-level heuristics can match this optimal expected runtime.}

\subsection{Mathematical Techniques}
We now introduce some important tools which will be used in the analyses.
\begin{theorem}\label{WaldEq}\textbf{Wald's Equation \citep{Wald1944,DoerrK13}}
Let $T$ be a random variable with bounded expectation and let $X_1, X_2, \dots$ be non-negative random variables with $E(X_i \mid T\geq i)\leq C$. Then
$$E\left(\sum_{i=1}^T X_i\right)\leq E(T)\cdot C.$$
\end{theorem}

\begin{theorem}\label{ADTheorem}\textbf{Additive Drift Theorem \citep{HeYao2001, OlivetoYao2011}}
Let $\{X_t\}_{t\geq0}$ be a Markov process over a finite set of states $S$, and $g:S\rightarrow\mathbb{R}_0^+$ a function that assigns a non-negative real number to every state. Let the time to reach the optimum be $T:=\min\{t\geq0\mid g(X_t)=0\}$. If there exists $\delta>0$ such that at any time step $t\geq0$ and at any state $X_t\geq0$ the following condition holds:
$$E(g(X_t)-g(X_{t+1})\mid g(X_t)>0)\geq\delta,$$
then
$$E(T\mid X_0)\leq\frac{g(X_0)}{\delta}$$
and
$$E(T)\leq\frac{E(g(X_0))}{\delta}.$$
\end{theorem}

\section{Hyper-Heuristics Are Faster Than Their Low-level Heuristics}\label{sec:leadingones}
In this section we show that \CHANGED{GRG} can outperform its constituent low-level heuristics even when the latter are efficient for the problem at hand.
\CHANGED{To achieve this, we equip the HH with the minimal set of low level heuristics: $H=\{\textsc{1BitFlip},\textsc{2BitFlip}\}$.}
 We first introduce a lower bound on the expected runtime of all algorithms which use only the \textsc{1BitFlip} and \textsc{2BitFlip} operators for \textsc{LeadingOnes}. \CHANGED{Afterwards we will show that GRG matches this expected runtime up to lower order terms.}
 
To achieve \CHANGED{the first} result we will rely on the following theorem.

\begin{theorem}[\cite{BoettcherEtAl2010,BuzdalovBuzdalova2015,Doerr2018Arxiv}]\label{BDNTheorem}
\CHANGED{Consider the run of an unbiased (1+1) black box algorithm optimizing the \textsc{LeadingOnes} function, and let T be the first time the optimum is generated. Given random initialisation:
\begin{align*}
E(T)=\frac{1}{2}\sum_{i=0}^{n-1} A_{i},
\end{align*}
}where $A_{i}$ is the expected time needed to find an improvement given a solution with
fitness $i$.
\end{theorem}

\CHANGED{Theorem~\ref{BDNTheorem} holds for any unbiased (1+1) black box algorithm for \textsc{LeadingOnes} \citep{Doerr2018Arxiv}. Since 
we analyse unbiased (1+1) black box algorithms, each differing by their mutation operator, we can apply Theorem~\ref{BDNTheorem} in our analysis.}

\CHANGED{$A_i$ is the expected time to find an improvement when the current solution $x$ consists of $i$ leading 1-bits. Hence,
the theorem implies that the overall expected runtime is minimised if the heuristic $h \in H$ with lowest expected improvement time $A_i$ is applied  whenever the current solution has $i$ leading 1-bits, for all $0\leq i <n$. Naturally, the heuristic with lowest expected time $A_i$ is the one with the highest {\it success} probability (i.e., the probability of producing a fitness improvement).

\begin{sloppypar}
The \textsc{1BitFlip} operator has a success probability of 
${P(\Imp_1\mid \textsc{LO}(x)=i)=\frac{1}{n}}$. The success probability of the
\textsc{2BitFlip} operator is ${P(\Imp_2\mid \textsc{LO}(x)=i)=\frac{1}{n}\cdot\frac{n-i-1}{n}\cdot 2 = \frac{2n-2i-2}{n^2}}$. Hence, for $i\geq\frac{n}{2}$, the \textsc{1BitFlip} operator is more effective (i.e., has a higher probability of success and hence a lower expected waiting time for each improvement $A_i$), while \textsc{2BitFlip} is preferable before. Note that both are equally effective when $\textsc{LO}(x)=\frac{n}{2}-1$. The expected runtime of an algorithm that applies these operators in such a way is a lower bound on the expected runtime of all unbiased (1+1) black box algorithms using only the same two operators. 
\end{sloppypar}
}

\begin{theorem}\label{ThmOpt}
The best-possible expected runtime of any \CHANGED{unbiased (1+1) black box} algorithm using only the \textsc{1BitFlip} and \textsc{2BitFlip} operators for \textsc{LeadingOnes} is $\frac{1+\ln(2)}{4}n^2+O(n)\approx0.42329n^2\CHANGED{+O(n)}$.
\end{theorem}

\begin{proof}
We  partition the analysis into two phases: \CHANGED{Phase 1}, where the \textsc{2BitFlip} operator is used for the first half of the search, and \CHANGED{Phase 2}, where the \textsc{1BitFlip} operator is used for the second half of the search. We consider the expected runtimes \CHANGED{of each phase separately (respectively,} $E(T_1)$ and $E(T_2)$)  and sum them to find the total expected runtime.

To find the expected runtime for \CHANGED{Phase 1} we can consider every point $x$ such that $\textsc{LO}(x)=i\geq \frac{n}{2}$ to be optimal \CHANGED{and we calculate the expected runtime for the \textsc{2BitFlip} operator to find any such point (i.e., $A_i=0$ for $i\geq\frac{n}{2}$)}. 
Since the probability of improvement of the \textsc{2BitFlip} operator is $P(\Imp_2\mid \textsc{LO}(x)=i)=\frac{2n-2i-2}{n^2}$, the expected time to make an improvement for $0\leq i<\frac{n}{2}$ is $A_{i}=\frac{n^2}{2n-2i-2}$. Hence, for Phase 1,
\begin{align*}
E(T_1)&=\frac{1}{2}\sum_{i=0}^{n-1}A_{i}=\frac{1}{2}\sum_{i=0}^{\frac{n}{2}-1}\frac{n^2}{2n-2i-2}=\frac{1}{2}\sum_{k=\frac{n}{2}}^{n-1}\frac{n^2}{2k}\\
&=\frac{n^2}{4}\sum_{k=\frac{n}{2}}^{n-1}\frac{1}{k}=\frac{n^2}{4}\Big(H_{n-1}-H_{\frac{n}{2}-1}\Big),
\end{align*}
\begin{sloppypar}
where $H_x$ is the $x_{th}$ Harmonic number, which can be bounded by
${\ln(n+1)+\gamma-\frac{1}{n+1}\leq H_n\leq \ln(n+1)+\gamma-\frac{1}{2(n+1)}}$
and $\gamma$ is the Euler-Mascheroni constant ($\gamma\approx 0.57722\dots$). Hence, we can bound the term $H_{n-1}-H_{\frac{n}{2}-1}$ from below by $\ln(2)$, and from above by $\ln(2)+\frac{3}{2n}$, which gives the following bounds on $E(T_1)$:
\end{sloppypar}
$$\ln(2)\cdot \frac{n^2}{4}\leq E(T_{1})\leq \Big(\ln(2)+\frac{3}{2n}\Big)\cdot
\frac{n^2}{4}.$$

Since \CHANGED{the upper bound on} $E(T_1)$ is at most $\frac{3n}{8}=O(n)$ greater than the lower bound of $\ln(2)\cdot \frac{n^2}{4}$, we have:
$$\frac{\ln(2)}{4}n^2\leq E(T_{1})\leq \frac{\ln(2)}{4}n^2+\frac{3n}{8}.$$

The expected runtime for Phase 2 can be calculated using Theorem \ref{BDNTheorem} with $A_{i}=n$ for $i \geq \frac{n}{2}$. Similarly to Phase 1, we consider $A_{i}=0$ for $i < \frac{n}{2}$.
$$E(T_{2})=\frac{1}{2}\sum_{i=0}^{n-1}A_{i}=\frac{1}{2}\sum_{i=\frac{n}{2}}^{n-1} n =\frac{1}{2}\cdot n\cdot \frac{n}{2}= \frac{n^2}{4}.$$

Hence, the total expected runtime is:
\begin{align*}
\frac{1+\ln(2)}{4}n^2\leq E(T_{1})+E(T_{2})\leq\frac{1+\ln(2)}{4}n^2+\frac{3n}{8}.
\end{align*}
\end{proof}

Thus, we have proven a theoretical lower bound of $\frac{1+\ln(2)}{4}n^2+O(n)\approx0.42329n^2\CHANGED{+O(n)}$ for the expected runtime of any \CHANGED{unbiased (1+1) black box} algorithm using the \textsc{1BitFlip} and \textsc{2BitFlip} operators \CHANGED{for \textsc{LeadingOnes}}, which is \CHANGED{better in the leading constant} than the expected runtime of \textsc{RLS} (i.e., $0.5n^2$).

The rest of the section is structured as follows. In Subsection~\ref{SubsectionUnderstanding} we show that the simple \CHANGED{HH} mechanisms all perform equivalently to Simple Random for \textsc{LeadingOnes}, up to lower order terms. In Subsection~\ref{subsectiongeneralised} we show that \CHANGED{GRG} is much faster and \CHANGED{matches the optimal expected runtime derived in Theorem~\ref{ThmOpt}, up to lower order terms.}

\subsection{Simple Mechanisms}\label{SubsectionUnderstanding}
In this section we show that the standard simple heuristic selection mechanisms \CHANGED{(i.e., Simple Random, Permutation, Greedy and Random Gradient)} all have the same expected runtime for \textsc{LeadingOnes}, up to lower order terms. 
\CHANGED{Note that, from the experiments performed by \cite{AlanaziLehre2014}, we know that these lower order terms do not have any visible impact on the average runtime already for small problem sizes.}

The following theorem derives the expected runtime for the Simple Random mechanism. The subsequent corollary extends the result to the other mechanisms.

\begin{theorem}\label{pTheorem}
Let \CHANGED{$p_1$} be the probability of choosing the \textsc{1BitFlip} mutation operator, and $1-\CHANGED{p_1}$ the probability of choosing the \textsc{2BitFlip} mutation operator. Then, the expected runtime of the Simple Random mechanism \CHANGED{using $H=\{\textsc{1BitFlip},\textsc{2BitFlip}\}$}, with $\CHANGED{p_1}\in(0,1)$, for \textsc{LeadingOnes} is $\frac{1}{4(1-\CHANGED{p_1})}\ln\left(\frac{2-\CHANGED{p_1}}{\CHANGED{p_1}}\right)n^2+o(n^2).$
If $\CHANGED{p_1}=0$ the expected runtime is infinite. If $\CHANGED{p_1}=1$, the expected runtime is $0.5n^2$.
\end{theorem}
\begin{proof}
If $\CHANGED{p_1}=0$, only the \textsc{2BitFlip} operator is used. There is a non-zero probability of reaching the point $x=1^{n-1}0$, which cannot be improved by using the \textsc{2BitFlip} operator. Hence, by the law of total probability, the expected runtime is infinite.

If $\CHANGED{p_1}=1$, only the \textsc{1BitFlip} operator is used, resulting in exactly RLS, which has expected runtime $0.5n^2$~\citep{BuzdalovBuzdalova2015}.
	
\begin{sloppypar}
If $\CHANGED{p_1}\in(0,1)$, \CHANGED{the expected runtime is} $\frac{1}{2}\sum_{i=1}^n A_{i}$, where $A_{i}$ is the expected time needed to find an improvement given a solution with fitness $i$ (Theorem~\ref{BDNTheorem}). In each iteration, the \textsc{1BitFlip} operator is chosen with probability $\CHANGED{p_1}$ and leads to a fitness improvement with probability \CHANGED{$P(\Imp_1 \mid \textsc{LO}(x)=i)=\frac{1}{n}$}; the \textsc{2BitFlip} operator is chosen with probability $1-\CHANGED{p_1}$ and leads to a fitness improvement with probability \CHANGED{$P(\Imp_2\mid\textsc{LO}(x)=i)=\frac{2n-2i-2}{n^2}$}. Hence,  \CHANGED{${(A_{i})^{-1}=p_1\cdot\frac{1}{n}+(1-p_1)\cdot\frac{2n-2i-2}{n^2}}$}, and
\end{sloppypar}
$$A_{i}=\frac{1}{\frac{\CHANGED{p_1}}{n}+\frac{(1-\CHANGED{p_1})(2n-2i-2)}{n^2}}=\frac{n^2}{2(1-\CHANGED{p_1})(n-i-1)+n\CHANGED{p_1}}.$$
\CHANGED{Hence, by Theorem~\ref{BDNTheorem}, the expected runtime is}
\begin{align}
&\phantom{=}\frac{1}{2}\sum_{i=0}^{n-1}\frac{n^2}{2(1-\CHANGED{p_1})(n-i-1)+n\CHANGED{p_1}} =\frac{n^2}{2}\sum_{k=1}^n\frac{1}{(2\CHANGED{p_1}-2) k+(2-\CHANGED{p_1}) n} \label{line1}\\
&=\frac{n^2}{2(2\CHANGED{p_1}-2)}\left(\ln\left(\frac{1+\frac{(2-\CHANGED{p_1})}{2\CHANGED{p_1}-2}}{\frac{(2-\CHANGED{p_1})}{2\CHANGED{p_1}-2}}\right)+o(1)\right) =\frac{n^2}{4(1-\CHANGED{p_1})}\ln\left(\frac{2-\CHANGED{p_1}}{\CHANGED{p_1}}\right)+o(n^2),\label{line2}
\end{align}
where Equation \ref{line1} becomes Equation \ref{line2} via the following simplification:
\begin{align*}
\sum_{k=1}^n\frac{1}{a\cdot k+b\cdot n}&=\frac{1}{a}\cdot\sum_{k=1}^n\frac{1}{k+\frac{b}{a}\cdot n}=\frac{1}{a}\left(\sum_{k=1}^{(1+\frac{b}{a})n}\frac{1}{k}-\sum_{k=1}^{\frac{b}{a}n}\frac{1}{k}\right)\\
&=\frac{1}{a}\left(\ln\left(\frac{1+\frac{b}{a}}{\frac{b}{a}}\right)+o(1)\right)
\end{align*}
with $a=2\CHANGED{p_1}-2$ and $b=2-\CHANGED{p_1}$.
\end{proof}

When $\CHANGED{p_1}=0.5$ (i.e., \CHANGED{there is an} equal chance of choosing each operator in each iteration), the standard Simple Random mechanism has an expected improvement time of $A_{i}=\frac{2n^2}{3n-2i-2}$, and \CHANGED{an} expected runtime \CHANGED{of} $\frac{\ln(3)}{2}n^2+o(n^2)\approx0.54931n^2\CHANGED{+o(n^2)}$.
The expected runtime improves with increasing $\CHANGED{p_1}$, hence the optimal choice is $\CHANGED{p_1}=1$ (i.e., RLS).

\begin{corollary}\label{CorPGRG}
The expected runtime of the Permutation, Greedy and Random Gradient mechanisms \CHANGED{using $H=\{\textsc{1BitFlip},\textsc{2BitFlip}\}$} for \textsc{LeadingOnes} is $\frac{\ln(3)}{2}n^2+o(n^2)\approx0.54931n^2\CHANGED{+o(n^2)}.$
\end{corollary}
\begin{proof}
\CHANGED{Given a bit-string $x$ with $\textsc{LO}(x)=i$,} let $p_i$ be the probability that \CHANGED{at least} one improvement is constructed within two fitness function evaluations. For the Permutation and Greedy mechanisms, we have $p_i=\frac{1}{n} + \left(1-\frac{1}{n}\right)\cdot\frac{2n-2i-2}{n^2}$ as either the \textsc{1BitFlip} operator can succeed in one iteration, or it can fail and then the \textsc{2BitFlip} operator can succeed in the next iteration. To \CHANGED{derive an upper bound on} the expected optimisation time, we note that the difference between $\frac{2}{p_i}$ (i.e., the expected waiting time for an improvement to be
constructed in terms of fitness evaluations) and the $A_{i}$ waiting times for the standard Simple Random mechanism \CHANGED{with $p_1=0.5$ (i.e., $A_{i}=\frac{2n^2}{3n-2i-2}$) is at most constant}, and thus the difference between the expected runtimes of these mechanisms and Simple Random is limited
to lower order terms. We note that with probability at most $\frac{2}{n^2}$, both mutations (the two mutations performed in parallel by the Greedy mechanism, or the two mutations performed sequentially by the Permutation mechanism) considered are improvements. As this occurs at most a
constant number of times in expectation (i.e., $O(n^2)\cdot \frac{2}{n^2}= O(1)$), and the maximum expected waiting time for any improving step is $O(n)$, the lower bound differs from the upper bound by at most an $O(1)\cdot O(n)= O(n)$ term. Thus, the expected runtime of these mechanisms is also $\frac{\ln(3)}{2}\cdot n^2+o(n^2)$.

For the Random Gradient mechanism, we note that the probability that an operator, when repeated following a success, is successful again is
at most $\frac{2}{n}$ (an upper bound on the success probability of the \textsc{2BitFlip} operator). Since there are at most $n$ improvements to be made throughout the search space, the expected number of \CHANGED{repetitions} which produce an improvement is at most $2$. If the chosen operator is not successful, the Random Gradient mechanism behaves identically to the Simple Random mechanism. Its expected runtime is therefore at least
the expected runtime of the Simple Random mechanism less an $O(n)$ term, and at most the expected runtime of the Simple Random mechanism plus $n$ (as there are at most $n$ iterations where the mechanisms differ in operator selection). Thus, its expected runtime is also $\frac{\ln(3)}{2}\cdot n^2+o(n^2)$.
\end{proof}

We point out that the lower bound for the Random Gradient \CHANGED{HH} contradicts the lower bound of $\frac{n^2}{9}(4+3\ln\left(\frac{10}{3}\right))\approx0.846n^2\CHANGED{+o(n^2)}$ found by \cite{AlanaziLehre2014}. However, their bound results from a small mistake in their proof. They should have found a lower bound of $\frac{n^2}{9}(3\ln\left(\frac{10}{3}\right))+o(n^2)\approx0.401n^2\CHANGED{+o(n^2)}$, which agrees with our \CHANGED{derived expected runtime} of $\frac{\ln 3}{2}n^2+o(n^2)\approx0.549n^2\CHANGED{+o(n^2)}$.

\subsection{Generalised Random Gradient}\label{subsectiongeneralised}
In this subsection we present a rigorous theoretical analysis of \CHANGED{GRG for \textsc{LeadingOnes}}. The main result of this subsection is that \CHANGED{GRG} is able to match, up to lower order terms, the best-possible performance of any algorithm using the \textsc{1BitFlip} and \textsc{2BitFlip} operators for \textsc{LeadingOnes}, as presented in Theorem~\ref{ThmOpt}. We \CHANGED{state} the main result in Corollary~\ref{CorGRGOpt} now, with the proof at the end of the subsection.
\begin{corollary}\label{CorGRGOpt}[Of Theorem~\ref{thm:LO-GenRandomGradient}]
The expected runtime of the Generalised Random Gradient hyper-heuristic \CHANGED{using $H=\{\textsc{1BitFlip},\textsc{2BitFlip}\}$} for \textsc{LeadingOnes}, with $\tau$ that satisfies both $\tau=\omega(n)$ and $\tau\leq\left(\frac{1}{2}-\varepsilon\right)n\ln(n)$, for some constant $\CHANGED{0<\varepsilon<\frac{1}{2}}$, is at most $\frac{1+\ln(2)}{4}n^2+o(n^2)\approx0.42329n^2\CHANGED{+o(n^2)}.$ 
\end{corollary}
This improves significantly upon the result of \cite{LissovoiEtAl2017}.  
\CHANGED{While the previous result showed that the HH outperformed its low level heuristics, Corollary~\ref{CorGRGOpt} also shows that the expected runtime matches the optimal one up to lower order terms.}
\CHANGED{Our previous work} only considered setting $\tau=cn$ for constant $c$, \CHANGED{a learning period for which optimal expected runtime could not be proven}.

\CHANGED{Before proving the corollary}, we first present the necessary prerequisite results. The following theorem is very general as it provides an upper bound on the expected runtime of \CHANGED{GRG} for any value of $\tau$ smaller than $\frac{1}{2}n\ln n$. In particular, \CHANGED{Theorem \ref{thm:LO-GenRandomGradient}} allows us to identify values of $\tau$ for which the expected runtime of the \CHANGED{HH} 
is the 
optimal \CHANGED{one 
achievable using the two operators.
This is the result} highlighted in Corollary~\ref{CorGRGOpt} and depicted in Figure~\ref{pic:upperbounds}. 

Our proof technique partitions the search space into $w$ \CHANGED{stages}, each representing an equal range of fitness values. \CHANGED{The expected runtime of the algorithm does not depend on $w$. However, $w$ does affect the upper bound on the runtime we obtain. 
The greater the number of \CHANGED{stages}, the tighter the upper bound the theorem statement provides as long as $w=o\left(\frac{n}{\exp(\tau/n)}\right)$.}
\begin{theorem} \label{thm:LO-GenRandomGradient}
\sloppy The expected runtime of the Generalised Random Gradient hyper-heuristic using \CHANGED{$H=\{\textsc{1BitFlip},\textsc{2BitFlip}\}$ for \textsc{LeadingOnes}, with} $\tau\leq\left(\frac{1}{2}-\varepsilon\right)n\ln(n)$, for some constant $\CHANGED{0<\varepsilon<\frac{1}{2}}$, is at most
$$\frac{n^2}{2}
\left(\sum_{j=1}^{w} \frac{\frac{2\tau}{n}+ e^{\frac{\tau}{n}}M_1(1)+ e^{2\frac{\tau}{n}\left(1-\frac{j}{w}\right)}M_2\left(
j,w\right)}{(e^\frac{\tau}{n}+e^{2\frac{\tau}{n}\left(1-\frac{j}{w}\right)}-2)w} \right) + o(n^2)$$
where \[ M_2(j,w) :=
\begin{cases}
\frac{\tau}{n} & \quad \text{if } j=w \text{ or }\frac{1}{2\left(1-\frac{j}{w}\right)}>\frac{\tau}{n}\\
\frac{1}{2\left(1-\frac{j}{w}\right)} & \quad \text{otherwise}\\
\end{cases}
\] and $M_1(x) := \min\!\left\{\frac{\tau}{n}, x\right\}$, with \CHANGED{$w\in\omega\left(\frac{\tau}{n}\right)\cap o\left(\frac{n}{\exp(\tau/n)}\right)$}.
\end{theorem}

Figure~\ref{pic:upperbounds} presents the theoretical upper bounds from Theorem~\ref{thm:LO-GenRandomGradient} for a range of linear $\tau$ values. For $\tau=5n$, GRG already outperforms RLS, giving an expected runtime of $0.46493n^2+o(n^2)$. \CHANGED{For $\tau=30n$, the performance improves to $0.42385n^2+o(n^2)$}, matching the best possible performance from Theorem~\ref{ThmOpt} up to 3 decimal places. We have seen that GRG is able to exactly match this \CHANGED{optimal} performance, up to lower order terms, for $\tau=\omega(n)$ in Corollary~\ref{CorGRGOpt}.
\begin{figure}[t]
\centering
	\begin{tikzpicture}
		\begin{axis}[
			xlabel=$\tau$ ($/n$), xmin=0, ymin=0.35, ymax=0.6,
			ylabel=Runtime $(/n^2)$,
			scaled ticks=false,tick label style={/pgf/number format/fixed},
			mark size=2pt,ymajorgrids=true,
			y tick label style={
				/pgf/number format/.cd,fixed,fixed zerofill,precision=3,
				/tikz/.cd
			},
			ytick={0.4,0.42329,0.5,0.6},
			width=0.98\linewidth,
			height=2.5in,
			legend pos=north east,legend columns=1,xmax=50
		]
		\addplot[red,mark=o]         file {Experiments/Fig1-3/GRG2OP.dat};
			\draw [thick, draw=teal]   (axis cs: 0,0.423286795) -- (axis cs: 50,0.423286795) -- node[below,align=left] {$2_\mathrm{Opt}$} ++(-90,0);
			\draw [thick, draw=brown]   (axis cs: 0,0.5) -- (axis cs: 50,0.5) -- node[above,align=right] {RLS} ++(-25,0);
		
		\legend{GRG}
		\end{axis}
	\end{tikzpicture}
	\caption{The leading constants in the theoretical upper bounds on the expected number of fitness function evaluations required by the Generalised Random Gradient hyper-heuristic to find the \textsc{LeadingOnes} optimum ($w=100,\!000$). \NEWCHANGED{$2_\mathrm{Opt}$ is the leading constant in the best expected runtime achievable using the \textsc{1BitFlip} and \textsc{2BitFlip} operators in any order.}}
	\label{pic:upperbounds}
\end{figure}
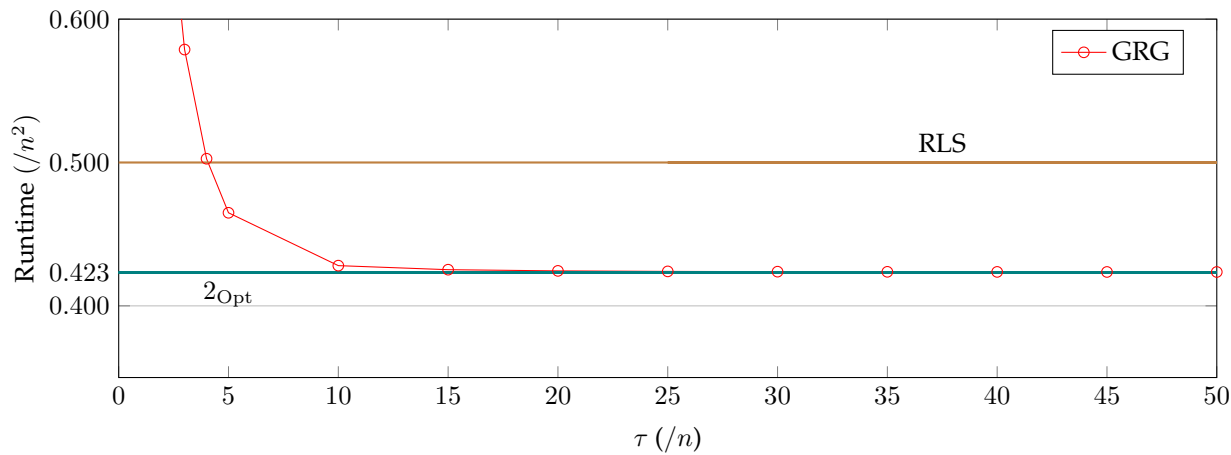

\begin{proof}[Of Theorem~\ref{thm:LO-GenRandomGradient}]

For the purpose of this proof, we partition the optimisation process into \CHANGED{$w$} \CHANGED{stages} based on the value of the \textsc{LeadingOnes} fitness function: during \CHANGED{stage} $j$, the \textsc{LeadingOnes} value of the current solution is at least $\frac{(j-1)n}{w}$ and less than
$\frac{j \cdot n}{w}$. After all the $w$ \CHANGED{stages} have been completed, the global optimum with a \textsc{LeadingOnes} value of $n$ will have been found.

\NEWCHANGED{We will provide an upper bound on the expectation of the runtime $T$ of GRG for \textsc{LeadingOnes} by summing up the expected values of $T_j$ (the expected time
between the first random operator choice in each stage of the optimisation process and the moment when the necessary fitness value to enter the following stage is reached).}
More precisely, our analysis of $E(T_j)$ requires the \CHANGED{stage} to start with a random
choice of mutation operator, we bound $E(T) \leq \sum_{j=1}^{w} \left(E(T_j) + E(S_{j+1}) \right)$
where $S_{j+1}$ is a random variable denoting the expected number of iterations
between the first solution in \CHANGED{stage} $j+1$ being constructed and the first random operator choice occurring in that \CHANGED{stage}. We will later show that with proper
parameter choices, the contribution of the $S_{j+1}$ terms can be bounded by
$o(n^2)$, and therefore they do not affect the leading constant in the overall bound.

Let us now consider $T_j$. 
Recall that
a mutation operator is selected uniformly at random and is allowed to run until it fails to produce an improvement within $\tau$ consecutive iterations. Let $N_j$ be a random variable denoting the number
of random operator choices the \CHANGED{HH} performs during \CHANGED{stage} $j$, and let $X_{j,1}, \ldots,
X_{j,N_j}$ be the number of iterations the \CHANGED{HH} runs each chosen operator for.
We note that $T_j = \sum_{k=1}^{N_j} X_{j,k}$ is a sum of a number
of non-negative variables, and we will later show that $E(N_j)$ is bounded; thus,
$E(T_j)$ can be bounded by applying Wald's equation (Theorem~\ref{WaldEq}):
\begin{equation} \label{eq:LO-GRD-Tj}
	E(T_j) = E\left(\textstyle{\sum_{k=1}^{N_j}} X_{j,k}\right) \leq E(N_j) E(X_{j}),
\end{equation}
where $E(X_{j})$ denotes an upper bound
on all $E(X_{j,k})$ in \CHANGED{stage} $j$.

To bound $E(N_j)$, the expected number of times the random operator selection is performed during \CHANGED{stage} $j$, we
\CHANGED{will provide a} lower bound on the expected number of improvements found following the operator selection and apply the
Additive Drift Theorem (Theorem~\ref{ADTheorem}) to find the expected number of random operator selections that occur before a
sufficient number of improvements \CHANGED{for entering} the next \CHANGED{stage} have been found.

Let $F_1$ and $F_2$ denote the events that the \textsc{1BitFlip} and \textsc{2BitFlip} operators fail to find an improvement during $\tau$ iterations. For the \textsc{1BitFlip} operator (with probability of fitness improvement in one iteration $\frac{1}{n}$),
this event occurs with probability $P(F_1) = \left(1 - \frac{1}{n}\right)^{\tau}$ throughout the process,
which is within $(e^{-\frac{\tau}{n}}-\frac{1}{n},e^{-\frac{\tau}{n}}]$ (since $\left(1-\frac{1}{n}\right)^n\leq \frac{1}{e}\leq\left(1-\frac{1}{n}\right)^{n-1}$). For the \textsc{2BitFlip} operator (with probability of fitness improvement in one iteration $\frac{2n-2i-2}{n^2}$), recall that during \CHANGED{stage}
$j$, the ancestor individual has at most $i=\frac{jn}{w}-1$ and at least $i=\frac{(j-1)n}{w}$ 1-bits. Thus:
\begin{align*}
P(F_2) &\leq \left(1 - 2\cdot\frac{1}{n}\cdot\frac{n-\left(\frac{jn}{w}-1\right)-1}{n}\right)^{\tau}\leq e^{-2\frac{\tau}{n}\left(1-\frac{j}{w}\right)},
\end{align*}
and
\begin{align*}
P(F_2) &\geq \left(1 - 2\cdot\frac{1}{n}\cdot\frac{n-\frac{(j-1)n}{w}-1}{n}\right)^{\tau} \\
&> \left(1 - \frac{2\cdot\left(1-\frac{j-1}{w}\right)}{n}\right)^{\tau} \geq e^{-\frac{2\tau}{n}\left(1-\frac{j-1}{w}\right)} - \frac{1}{n}.
\end{align*}

A geometric distribution with parameter $p = P(F_1)$ (or $p=P(F_2)$) can be used to model the number of improvements that the \textsc{1BitFlip} (or \textsc{2BitFlip}) operator finds prior to failing to find an improvement for $\tau$ iterations. The expectation of this distribution is $\frac{1-p}{p} = \frac{1}{p} - 1$. We can model these events with a geometric random variable since GRG immediately restarts a period of $\tau$ iterations with the chosen heuristic when a successful mutation occurs (i.e., it sets $c_t=0$ \CHANGED{in Algorithm~\ref{alg:GRG}}). Hence, there are no periods of $\tau$ consecutive unsuccessful mutations \CHANGED{apart from the period when} the operator finally fails.
Combined over both operators, the expected number of improvements $D_j$
produced following a single random operator selection during \CHANGED{stage} $j$ is:
$$E(D_j)=\frac{1}{2}\left(\frac{1}{P(F_1)}-1\right) +
\frac{1}{2}\left(\frac{1}{P(F_2)}-1\right)
 \geq \frac{e^{\frac{\tau}{n}}}{2} + \frac{e^{\frac{2\tau}{n}\left(1-\frac{j}{w}\right)}}{2} - 1,$$
\CHANGED{where we inserted} the upper bounds on $P(F_1)$ and $P(F_2)$ \CHANGED{to obtain the inequality}. We will use this expectation as the drift on the progress of the randomly chosen operator
in the Additive Drift Theorem to provide an upper bound on $E(N_j)$. Recall that each \CHANGED{stage} requires advancing
through at most $\frac{n}{w}$ fitness values. Since bits beyond the leading \CHANGED{1-bit} prefix and the first \CHANGED{0-bit} remain uniformly distributed, at most $\frac{n}{2w}$ improvements by mutation are required in expectation. If each step
of a random process in expectation contributes $E(D_j)$ improvements by mutation, then the expected
number of steps required to complete \CHANGED{stage} $j$ is at most
\begin{align} \label{eq:grd-stage-drift}
	E(N_j) \leq \frac{\frac{n}{2w}}{E(D_j)} \leq \frac{n}{\left(e^\frac{\tau}{n}+e^{\frac{2\tau}{n}\left(1-\frac{j}{w}\right)}-2\right)w},
\end{align}
by the Additive Drift Theorem.

We apply Wald's equation to bound $E(X_{j})$, the expected number of iterations before a selected mutation operator fails to produce an improvement for $\tau$ iterations: let $S$ be the number of improvements found by the operator before it fails, and $W_1, \ldots, W_S$ be the number of
iterations it took to find each of those improvements; then, once selected, the \textsc{1BitFlip}
operator runs for
$$E(X_{j} \mid \textsc{1BitFlip}) = \tau + \sum_{k=1}^{S} E(W_k) = \tau + E(S)E(W_1 \mid S \geq 1,
\textsc{1BitFlip}),$$
where $\tau$ accounts for the iterations immediately before failure and the sum \CHANGED{counts the expected number of} iterations
preceding each constructed improvement. Recall that ${E(S) = \frac{1}{P(F_1)} - 1}$ by the properties of the
geometric distribution and observe that ${E(W_1 \mid S \geq 1) = E(W_1 \mid W_1 \leq \tau) \leq
\min\{\tau, E(W_1)\}}$. Using a waiting time argument (i.e., considering the expectation of a geometric random variable) gives $E(W_1)=n$, and with a similar argument for the \textsc{2BitFlip} operator, $E(W_2)\leq \frac{n^2}{2\left(n-\left(\frac{jn}{w}-1\right)-1\right)}=\frac{n}{2\left(1-\frac{j}{w}\right)}$. Hence:
\begin{align*}
	E(X_{j} \mid \textsc{1BitFlip}) &\leq \tau + \left(\frac{1}{P(F_1)} - 1\right)\min\{\tau, n\}\\
	E(X_{j} \mid \textsc{2BitFlip}) &\leq \tau + \left(\frac{1}{P(F_2)} - 1\right)
		\min\left\{\tau, \frac{n}{2\left(1-\frac{j}{w}\right)}\right\} .
\end{align*}
Combining these conditional expectations with lower bounds on $P(F_1)$ and $P(F_2)$ yields
\begin{align} \label{eq:grd-op-phase-length}
	E(X_{j}) & \leq \frac{1}{2} \cdot E(X_{j} \mid \textsc{1BitFlip})
	+ \frac{1}{2} \cdot E(X_{j} \mid \textsc{2BitFlip}) \\
	& \leq \tau
		+ \frac{\min\{\tau, n\}}{2\left(e^{-\frac{\tau}{n}}-\frac{1}{n}\right)}
		+ \frac{\min\left\{\tau, \frac{n}{2\left(1-\frac{j}{w}\right)}\right\}}{2\left(e^{-\frac{2\tau}{n}\left(1-\frac{j-1}{w}\right)}-\frac{1}{n}\right)}\nonumber\\
& \leq \frac{n}{2}\left( \frac{2\tau}{n}
		+ \frac{\min\{\frac{\tau}{n}, 1\}}{e^{-\frac{\tau}{n}}-\frac{1}{n}}
		+ \frac{\min\left\{\frac{\tau}{n}, \frac{1}{2\left(1-\frac{j}{w}\right)}\right\}}{e^{-\frac{2\tau}{n}\left(1-\frac{j-1}{w}\right)}-\frac{1}{n}}
		\right) \nonumber\\
	& \leq \frac{n}{2}\left( \frac{2\tau}{n}
		+ e^{\frac{\tau}{n}}\min\left\{\frac{\tau}{n}, 1\right\} + 
		e^{\frac{2\tau}{n}\left(1-\frac{j}{w}\right)}\min\left\{\frac{\tau}{n}, \frac{1}{2\left(1-\frac{j}{w}\right)}\right\}
		\right) + O(1),\nonumber
\end{align}
using \CHANGED{$e^{\frac{2\tau}{n}\left(1-\frac{j-1}{w}\right)}=e^{\frac{2\tau}{wn}}\cdot e^{\frac{2\tau}{n}\left(1-\frac{j}{w}\right)}=(1+o(1))\cdot e^{\frac{2\tau}{n}\left(1-\frac{j}{w}\right)}$, since $w=\omega\left(\frac{\tau}{n}\right)$,} and $\frac{a}{b-\frac{c}{n}} = \frac{a}{b}+ \frac{ac}{b^2n - bc} = \frac{a}{b}+ O\left(\frac{1}{n}\right)$ to limit the contributions of the $-\frac{1}{n}$ terms in the denominators to lower order terms.

If $j=w$, then $\min\!\left\{\frac{\tau}{n}, \frac{1}{2\left(1-\frac{j}{w}\right)}\right\}$ is undefined. However, we know that the chosen operator would only be applied for a maximum of $\tau$ steps. We thus refer to the functions $M_1(x)=\min\!\left\{\frac{\tau}{n},x\right\}$ and $M_2(j,w)$ from now on, where \[ M_2(j,w) :=
\begin{cases}
\frac{\tau}{n} & \quad \text{if } j=w \text{ or }\frac{1}{2\left(1-\frac{j}{w}\right)}>\frac{\tau}{n}\\
\frac{1}{2\left(1-\frac{j}{w}\right)} & \quad \text{otherwise.}\\
\end{cases}
\]

Finally, we can bound the overall expected runtime of the \CHANGED{HH}. Recall that $S_j$ denotes the number of iterations the \CHANGED{HH} spends in \CHANGED{stage} $j$ while using the random operator chosen during \CHANGED{stage} $j-1$,
and since $E(S_j) \leq \max\left(E(X_{j} \mid \textsc{1BitFlip}), E(X_{j} \mid \textsc{2BitFlip})\right) < 2 E(X_{j})= O\left(n\cdot\exp\left(\frac{2\tau}{n}\right)\right)$ and \CHANGED{$w= o\left(\frac{n}{\exp\left(\frac{2\tau}{n}\right)}\right)$}, \CHANGED{it follows that} $\sum_{j=1}^{w} E(S_{j+1}) = o(n^2)$. 

Substituting the bounds provided in Eq~\ref{eq:grd-stage-drift} and Eq.~\ref{eq:grd-op-phase-length} into Eq.~\ref{eq:LO-GRD-Tj} yields the theorem statement:
\begin{align*}
	E(T) & \leq \sum_{j=1}^{w} \left(E(T_j) + E(S_{j+1}) \right) \leq \left(\sum_{j=1}^{w} E(N_j) E(X_{j})\right) + o(n^2) \\
	& \leq o(n^2) + \frac{n^2}{2} \times \sum_{j=1}^{w} \frac{
		\frac{2\tau}{n}
		+ e^{\frac{\tau}{n}}M_1(1)
		+ e^{\frac{2\tau}{n}\left(1-\frac{j}{w}\right)}M_2\left(j,w\right)
	}{\left(e^\frac{\tau}{n}+e^{\frac{2\tau}{n}\left(1-\frac{j}{w}\right)}-2\right)w} .
\end{align*}
\end{proof}

We can now prove the main result of this section (i.e., Corollary~\ref{CorGRGOpt}), \CHANGED{which states that the expected runtime of GRG using $H=\{1\textsc{BitFlip},2\textsc{BitFlip}\}$ for \textsc{LeadingOnes} with $\tau$ that satisfies both $\tau=\omega(n)$ and $\tau\leq\left(\frac{1}{2}-\varepsilon\right)n\ln(n)$, for some constant $0<\varepsilon<\frac{1}{2}$, is at most $\frac{1+\ln(2)}{4}n^2+o(n^2)\approx042329n^2+o(n^2)$}.
\begin{proof}[Of Corollary~\ref{CorGRGOpt}]

Consider first the terms $M_1(1)$ and $M_2(j,w)$. Since $\tau=\omega(n)$, we have $M_1(1)=1$. For $M_2(j,w)$ note that for the first half of the search, $M_2(j,w)=\frac{1}{2\left(1-\frac{j}{w}\right)}$, and in the second half of the search, $M_2(j,w)= O(w)=n^{o(1)}$ (excluding the case when $j=w$ in which $M_2(j,w)=\frac{\tau}{n}=n^{o(1)}$). Note \CHANGED{also} that $\frac{\tau}{n}=\omega(1)$. Hence\CHANGED{, we can simplify the sum from Theorem~\ref{thm:LO-GenRandomGradient} as follows}:
\begin{align*}
E(T)&\leq\frac{n^2}{2}\left(\sum_{j=1}^{w} \frac{\frac{2\tau}{n}+ e^{\frac{\tau}{n}}M_1(1)+ e^{\frac{2\tau}{n}\left(1-\frac{j}{w}\right)}M_2\left(j,w\right)}{(e^\frac{\tau}{n}+e^{2\frac{\tau}{n}\left(1-\frac{j}{w}\right)}-2)w} \right) + o(n^2)\\
&=\frac{n^2}{2}\left(\sum_{j=1}^{w} \frac{\frac{2\tau}{n}+ e^{\frac{\tau}{n}}+ e^{\frac{2\tau}{n}\left(1-\frac{j}{w}\right)}M_2\left(j,w\right)}{\left(e^\frac{\tau}{n}+e^{2\frac{\tau}{n}\left(1-\frac{j}{w}\right)}-2\right)w} \right) + o(n^2)\\
&=\frac{n^2}{2}\left(\sum_{j=1}^{w} \frac{1}{w}\left(\frac{\frac{2\tau}{n}+2}{e^\frac{\tau}{n}+e^{2\frac{\tau}{n}\left(1-\frac{j}{w}\right)}-2}+\frac{e^{\frac{\tau}{n}}+ e^{\frac{2\tau}{n}\left(1-\frac{j}{w}\right)}M_2\left(j,w\right)-2}{e^\frac{\tau}{n}+e^{2\frac{\tau}{n}\left(1-\frac{j}{w}\right)}-2} \right)\right)\\
&\phantom{=} + o(n^2)\\
&=\frac{n^2}{2}\left(\sum_{j=1}^{w} \frac{1}{w}\left(\frac{e^{\frac{\tau}{n}}+ e^{\frac{2\tau}{n}\left(1-\frac{j}{w}\right)}M_2\left(j,w\right)-2}{e^\frac{\tau}{n}+e^{2\frac{\tau}{n}\left(1-\frac{j}{w}\right)}-2} \right)\right) + o(n^2).
\end{align*}

The exponential terms will \CHANGED{asymptotically} dominate the summation term (since $\frac{\tau}{n}=\omega(1)$). When $j>\frac{w}{2}$, we note that $\exp\left(\frac{\tau}{n}\right)>\exp\left(\frac{2\tau}{n}\left(1-\frac{j}{w}\right)\right)$ and this term will dominate the other (vice-versa for $j<\frac{w}{2}$). At $j=\frac{w}{2}$, the two values are equal (\CHANGED{with the term in the summand equal to 1}). Hence, we can split the sum into three parts:
\begin{align*}
E(T)&\leq\frac{n^2}{2}\left(\sum_{j=1}^{w} \frac{1}{w}\left(\frac{e^{\frac{\tau}{n}}+ e^{\frac{2\tau}{n}\left(1-\frac{j}{w}\right)}M_2\left(j,w\right)-2}{e^\frac{\tau}{n}+e^{2\frac{\tau}{n}\left(1-\frac{j}{w}\right)}-2} \right)\right) + o(n^2)\\
&=\frac{n^2}{2}\left(\sum_{j=1}^{w/2-1} \frac{1}{w}\left(\frac{e^{\frac{2\tau}{n}\left(1-\frac{j}{w}\right)}\cdot\frac{1}{2\left(1-\frac{j}{w}\right)}}{e^{\frac{2\tau}{n}\left(1-\frac{j}{w}\right)}} \right)\right)+\frac{n^2}{2w}+\frac{n^2}{2}\left(\sum_{j=w/2+1}^{w} \frac{1}{w}\left(\frac{e^{\frac{\tau}{n}}}{e^{\frac{\tau}{n}}}\right)\right)\\
&\phantom{=} + o(n^2)\\
&=\frac{n^2}{4}\left(\sum_{j=1}^{w/2} \frac{1}{w-j}\right)+\frac{n^2}{4}+o(n^2)=\left(\frac{1+\ln(2)}{4}\right)n^2+o(n^2).
\end{align*}
\begin{sloppypar}
Note that at $j=w$, the respective `dominating terms' are $\exp\left(\frac{\tau}{n}\right)$ and ${\exp\left(\frac{2\tau}{n}\left(1-\frac{w}{w}\right)\right)=1}$, and $M(j,w)=\frac{\tau}{n}<\frac{1}{2}\ln(n)$. Hence, $\exp\left(\frac{\tau}{n}\right)$ will dominate.
\end{sloppypar}
\end{proof}

\section{Increasing the Choice of Low-level Heuristics Leads to Improved Performance}\label{sec:MoreThanTwo}
Runtime analyses of \CHANGED{HHs} are easier if the algorithms can only choose between two operators. However, in realistic contexts, \CHANGED{HHs} have to choose between many more operators. 

In this section, we consider the previously analysed \CHANGED{GRG HH}, yet extend the set $H$ of low-level heuristics to be of size $k\geq2$ i.e., $|H|=k=\Theta(1)$. We extend upon the previous analysis by considering the $k$ operators as $k$ different mutation operators, each flipping between $1$ and $k$ bits with replacement uniformly at random. This more accurately represents \CHANGED{HH} approaches employed for real-world problems.

As before, we consider the \textsc{LeadingOnes} benchmark function. We will rigorously show that the performance of the simple mechanisms deteriorates as the number ($|H|=k=\Theta(1)$) of operators increases. In addition, we prove decreasing upper bounds on the expected runtime of \CHANGED{GRG} as the number of operators increases. 

The main result of this section is the following theorem, \CHANGED{which states that GRG} that chooses between $k$ stochastic mutation operators has better expected performance than any algorithm, including the best-possible, using less than $k$ stochastic mutation operators for \textsc{LeadingOnes}. We present the statement of Theorem~\ref{besttheorem} now, and the proof at the end of this section.
\begin{theorem}\label{besttheorem}
\CHANGED{The expected runtime for \textsc{LeadingOnes} of the Generalised Random Gradient hyper-heuristic using $H=\{1\textsc{BitFlip},\dots,k\textsc{BitFlip}\}$ and $k=\Theta(1)$, with $\tau$ that satisfies both $\tau=\omega(n)$ and $\tau\leq\left(\frac{1}{k}-\varepsilon\right)n\ln(n)$ for some constant $\CHANGED{0<\varepsilon<\frac{1}{k}}$, is smaller than the best-possible expected runtime of any unbiased (1+1) black box algorithm using any strict subset of $\{1\textsc{BitFlip},\dots,k\textsc{BitFlip}\}$.}
\end{theorem}

\CHANGED{Theorem~\ref{besttheorem}} highlights the power of \CHANGED{HHs} as general-purpose problem solvers. The inclusion of more heuristics to the set of low-level heuristics is implied to be preferable, showcasing the impressive learning capabilities of even simple \CHANGED{HHs}.  \CHANGED{Figure~\ref{pic:upperbounds2v2} highlights the meaning of Theorem~\ref{besttheorem} for $k=3$ and $k=5$.} \CHANGED{The expected runtime of GRG with three operators is better than the best possible expected runtime achievable using the first two operators in any combination. The figure also highlights that with five operators the expected runtime is better than any one achievable using fewer operators.}

\begin{figure}[t]
	\begin{tikzpicture}
		\begin{axis}[
			xlabel=$\tau$ ($/n$), xmin=0, ymin=0.40, ymax=0.44,
			ylabel=Runtime $(/n^2)$,
			scaled ticks=false,tick label style={/pgf/number format/fixed},
			mark size=2pt,ymajorgrids=true,
			y tick label style={
				/pgf/number format/.cd,fixed,fixed zerofill,precision=2,
				/tikz/.cd
			},
			width=0.485\linewidth,
			height=2in,
			legend pos=north east,legend columns=1,xmax=50
		]
		\addplot[blue,mark=o]         file {Experiments/Fig1-3/GRG2OP.dat};
		\addplot[red,mark=*]          file {Experiments/Fig1-3/GRG3OP.dat};
		\draw [thick, draw=teal]   (axis cs: 0,0.423286795) -- (axis cs: 50,0.423286795) -- node[below,align=left] {$2_\mathrm{Opt}$} ++(-90,0);
		
		\legend{$k=2$, $k=3$}
		\end{axis}
	\end{tikzpicture}\hfill
		\begin{tikzpicture}
		\begin{axis}[
			xlabel=$\tau$ ($/n$), xmin=0, ymin=0.38, ymax=0.44,
			ylabel=Runtime $(/n^2)$,
			scaled ticks=false,tick label style={/pgf/number format/fixed},
			mark size=2pt,ymajorgrids=true,
			y tick label style={
				/pgf/number format/.cd,fixed,fixed zerofill,precision=2,
				/tikz/.cd
			},
			width=0.485\linewidth,
			height=2in,
			legend pos=north east,legend columns=1,xmax=50
		]
		\addplot[green,mark=square]        file {Experiments/Fig1-3/GRG4OP.dat};
		\addplot[orange,mark=diamond]        file {Experiments/Fig1-3/GRG5OP.dat};
		\draw [thick, draw=teal]   (axis cs: 0,0.3983094070) -- (axis cs: 50,0.3983094070) -- node[below,align=left] {$4_\mathrm{Opt}$} ++(-90,0);
		
		\legend{$k=4$, $k=5$}
		\end{axis}
	\end{tikzpicture}
	\caption{A comparison of the optimal expected runtimes of the Generalised Random Gradient hyper-heuristic with $k$-operators against the leading constant in the theoretical upper bound of the expected runtime of the Generalised Random Gradient hyper-heuristic with $k+1$-operators; for $k=2$ and $k=4$. $2_\mathrm{Opt}$ ($\approx0.42329n^2+o(n^2)$) and $4_\mathrm{Opt}$ ($\approx0.39830n^2+o(n^2)$) are the best possible runtimes achievable by the GRG hyper-heuristic with access to $2$ and $4$ low-level heuristics respectively, from Corollary~\ref{corkoptbest}.}
	\label{pic:upperbounds2v2}
\end{figure}
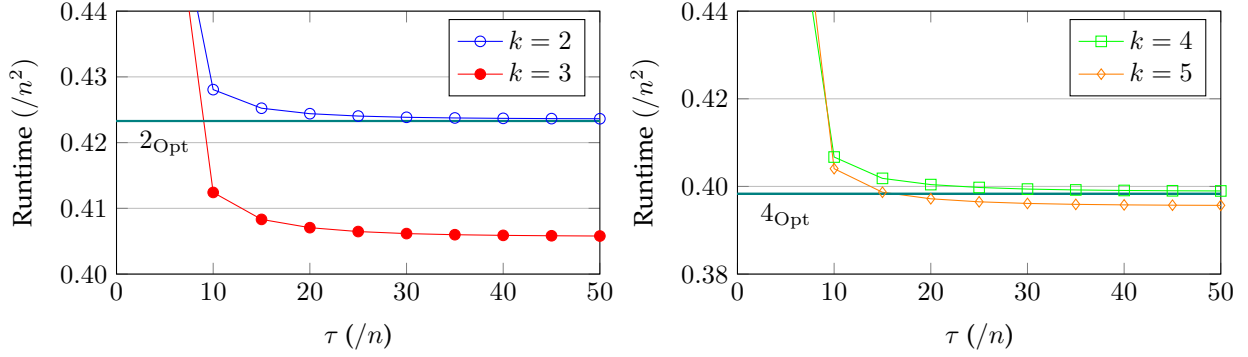

We now present the necessary prerequisite analysis to \CHANGED{obtain} the result. We first derive the best-possible expected runtime achievable by any \CHANGED{unbiased (1+1) black box algorithm using $\{1\textsc{BitFlip},\dots,k\textsc{BitFlip}\}$ mutation operators for \textsc{LeadingOnes}}. Before we prove this, we introduce the following two helpful lemmata \CHANGED{(which hold for all problem sizes $n\geq1$)}.

\begin{lemma}\label{moplemma}
\CHANGED{Given a bit-string $x$ with $LO(x)=i$, the probability of improvement of a mutation operator \CHANGED{that flips} $m=\Theta(1)$ bits with replacement ($m\textsc{BitFlip}$) \NEWCHANGED{is}}
$$P(\Imp_m\mid \textsc{LO}(x)=i)=m\cdot\frac{1}{n}\cdot\left(\frac{n-i-1}{n}\right)^{m-1}+O\left(\frac{1}{n^2}\right).$$
\end{lemma}
\begin{proof}
\NEWCHANGED{In the context of the \textsc{LeadingOnes} function, the way for the $m\textsc{BitFlip}$ operator to make an improvement is to flip the first 0-bit after the leading 1-bits, whilst keeping the leading 1-bits in the prefix unchanged. This occurs with probability $m\cdot\frac{1}{n}\cdot\left(\frac{n-i-1}{n}\right)^{m-1}$.

It is also possible to make fitness improvements while one or more bits of the prefix (i.e., within the first $i+1$ bits) are flipped, each an even number of times. This can occur by flipping and reflipping between 1 and $\left\lfloor\frac{m-1}{2}\right\rfloor$ bits. Firstly, the first 0-bit must be flipped into a 1-bit, which occurs with probability at most $\frac{m}{n}$. Then, an even number of the remaining bitflips can be used to flip bits in the prefix, which need to be reflipped to maintain the prefix. The remaining flips will occur within the suffix $n-i-1$ bits and these events can occur in any order. Hence, the probability of producing an improvement while flipping and reflipping an even number of bits ($>0$) in the prefix is at most:
\begin{align*}
&\phantom{=}\frac{m}{n}\cdot\sum_{j=1}^{\left\lfloor\frac{m-1}{2}\right\rfloor}\left[\binom{m-1}{2j}\cdot\frac{1}{i+1}\cdot\left(\frac{i+1}{n-1}\right)^{2j}\cdot\left(\frac{n-i-1}{n-1}\right)^{m-2j-1}\right]\\
&=\frac{m}{n}\cdot\binom{m-1}{2}\cdot\frac{1}{i+1}\cdot\left(\frac{i+1}{n-1}\right)^2\cdot\left(\frac{n-i-1}{n-1}\right)^{m-3}+O\left(\frac{1}{n^2}\right)=O\left(\frac{1}{n^2}\right),
\end{align*}
given $m=\Theta(1)$.  Hence, the probability of improvement of the $m\textsc{BitFlip}$ operator for \textsc{LeadingOnes} is}
$$P(\Imp_m\CHANGED{\mid \textsc{LO}(x)=i})=m\cdot\frac{1}{n}\cdot\left(\frac{n-i-1}{n}\right)^{m-1}+O\left(\frac{1}{n^2}\right).$$
\end{proof}

\begin{lemma}\label{moplemma2}
\CHANGED{Given a bit-string $x$ with $\textsc{LO}(x)=i$, consider two mutation operators flipping $a=\Theta(1)$ and $b=\Theta(1)$ bits with replacement respectively ($a\textsc{BitFlip}$ and $b\textsc{BitFlip}$), with $b>a$. Then (excluding $i=n-1$) $b\textsc{BitFlip}$ has a higher probability of success than $a\textsc{BitFlip}$ (i.e., $P(\Imp_b\mid\textsc{LO}(x)=i)>P(\Imp_a\mid\textsc{LO}(x)=i)$) when} 
$$i<n\left(1-\left(\frac{b}{a}\right)^{\frac{1}{a-b}}\right)-1.$$
\end{lemma}

\CHANGED{Lemma~\ref{moplemma2}} follows from Lemma~\ref{moplemma}. In particular, an operator that flips $m$ bits \CHANGED{(i.e., $m\textsc{BitFlip}$) has a higher probability of success than} an operator that flips $m-1$ bits \CHANGED{(i.e., $m-1\textsc{BitFlip}$) or less} when $i<\frac{n}{m}-1$. It is worth noting that only an operator which flips an odd number of bits can make progress when $\textsc{LO}(x)=n-1$, and the \textsc{1BitFlip} operator has the best success probability at this point \CHANGED{(i.e., $P(\Imp_1\mid\textsc{LO}(x)=i)=\frac{1}{n}$). We additionally note through a simple calculation that the \textsc{1BitFlip} operator has the highest probability of success in the second half of the search space (i.e., when $\frac{n}{2}\leq i<n$).}

In Theorem~\ref{koptbestcase} we present the best-possible expected runtime for any unbiased (1+1) black box algorithm using $\CHANGED{\{1\textsc{BitFlip},\dots,k\textsc{BitFlip}\}}$ as mutation operators for \textsc{LeadingOnes}. Similar results have recently been presented by \cite{Doerr2018Arxiv} and \cite{DoerrWagner2018} for mutation operators that flip bits \CHANGED{without} replacement. The \CHANGED{expected runtimes} are the same up to lower order terms.

\begin{theorem}\label{koptbestcase}
\sloppy The best-possible expected runtime  of any unbiased (1+1) black box algorithm using \CHANGED{$\{1\textsc{BitFlip},\dots,k\textsc{BitFlip}\}$ and $k=\Theta(1)$} for \textsc{LeadingOnes} is
$$E(T_{k,\mathrm{Opt}})=\frac{1}{2}\left(\sum_{i=0}^{\frac{n}{k}-1}\frac{1}{k\cdot\frac{1}{n}\cdot\left(\frac{n-i-1}{n}\right)^{k-1}}+\sum_{m=1}^{k-1}\sum_{i=\frac{n}{m+1}}^{\frac{n}{m}-1}\frac{1}{m\cdot\frac{1}{n}\cdot\left(\frac{n-i-1}{n}\right)^{m-1}}\right)\pm o(n^2).$$
\end{theorem}
\begin{proof}
\CHANGED{We refer to an operator as optimal if it yields the highest probability of success based on the current $\textsc{LO}(x)=i$ value.} From Lemma~\ref{moplemma2} the \CHANGED{$m\textsc{BitFlip}$} operator is optimal when $\frac{n}{m+1}\leq i \leq \frac{n}{m}-1$ (unless $m=k$, in which case \CHANGED{$k\textsc{BitFlip}$} is optimal for $0\leq i \leq \frac{n}{k}-1$). 

\CHANGED{The best-possible expected runtime is achieved by an algorithm that always uses the optimal operator. Such an algorithm will apply the \textsc{1BitFlip} operator while $LO(x)\geq\frac{n}{2}$ and the other operators will have $P(\Imp_m\mid \textsc{LO}(x)=i)=\Theta\left(\frac{1}{n}\right)$ in their optimal regions. Hence, any $O\left(\frac{1}{n^2}\right)$ terms from the $P(\Imp_m\mid \textsc{LO}(x)=i)$ terms (for $2\leq m\leq k$) are asymptotically insignificant and their contributions can be grouped into a lower order $o(n^2)$ term in the final runtime.}

Hence, by standard waiting time arguments, coupled with Theorem \ref{BDNTheorem}, we have that the best-possible expected runtime when using $\CHANGED{H=\{1\textsc{BitFlip},\dots,k\textsc{BitFlip}\}}$ is
\begin{align*}
E(T_{k,\mathrm{Opt}})&=\frac{1}{2}\left(\sum_{i=0}^{\frac{n}{k}-1}\frac{1}{k\cdot\frac{1}{n}\cdot\left(\frac{n-i-1}{n}\right)^{k-1}}+\sum_{m=1}^{k-1}\sum_{i=\frac{n}{m+1}}^{\frac{n}{m}-1}\frac{1}{m\cdot\frac{1}{n}\cdot\left(\frac{n-i-1}{n}\right)^{m-1}}\right)\pm o(n^2).
\end{align*}
\end{proof}

In particular, taking limits as $n\rightarrow\infty$, we have $E(T_{1,\mathrm{Opt}})=\frac{1}{2}n^2$, $E(T_{2,\mathrm{Opt}})=\frac{1+\ln(2)}{4}n^2\approx0.42329n^2$,
 $E(T_{3,\mathrm{Opt}})=\left(\frac{1}{3}+\frac{\ln(2)}{2}-\frac{\ln(3)}{4}\right)n^2\approx 0.40525n^2$, $E(T_{5,\mathrm{Opt}})=\left(\frac{3721}{11520}+\frac{\ln(2)}{2}-\frac{\ln(3)}{4}\right)n^2\approx 0.39492n^2$. A \CHANGED{closed form} result for $E(T_{k,\mathrm{Opt}})$ is difficult to find as is the limit for the best-possible expected runtime as $k\rightarrow\infty$. A numerical analysis by \cite{DoerrWagner2018} suggests \CHANGED{an expected runtime} of $E(T_{\infty,\mathrm{Opt}})\approx0.388n^2\pm o(n^2)$.

To emphasise the practical importance of our result, we first prove that the performance of simple mechanisms considered in the literature \citep{CowlingEtAl2000,CowlingEtAl2002,AlanaziLehre2014} \CHANGED{worsens with the increase of the low-level heuristic set size.}

\subsection{`Simple' Mechanisms}\label{Section:SimpleMore}
We will now see how the simple learning mechanisms (Simple Random, Permutation, Greedy and Random Gradient) perform when having to choose between $k$ operators (i.e., $\CHANGED{H=\{1\textsc{BitFlip},\dots,k\textsc{BitFlip}\}}$) \CHANGED{for \textsc{LeadingOnes}}. We will show that incorporating more operators is detrimental to the performance of the simple mechanisms for \textsc{LeadingOnes}. 

We start again by stating the expected runtime of the Simple Random mechanism, and use this as a basis for the other three mechanisms. Recall that the standard Simple Random mechanism chooses each operator uniformly at random in each iteration (i.e., with probability $\frac{1}{k}$ when using $k$ operators). 

\begin{theorem}\label{ThmSRk}
The expected runtime of the Simple Random mechanism using \CHANGED{$H=\{1\textsc{BitFlip},\dots,k\textsc{BitFlip}\}$ and $k=\Theta(1)$ for \textsc{LeadingOnes}} is
$$\frac{k}{2}\cdot\sum_{i=0}^{n-1}\frac{1}{\sum_{m=1}^k m\cdot\frac{1}{n}\cdot\left(\frac{n-i-1}{n}\right)^{m-1}}- o(n^2).$$
\end{theorem}
Note that the expected runtime of the Simple Random mechanism increases with $k$. Hence, incorporating more operators is detrimental to its performance. In particular, the expected runtimes when using 1, 2 and 3 operators are \CHANGED{respectively (up to lower order terms)} $\frac{1}{2}n^2$, $\frac{ln(3)}{2}n^2\approx0.54931n^2$ and $\left(\frac{3\sqrt2}{4}\arctan\left(\frac{\sqrt2}{2}\right)\right)n^2\approx0.65281n^2$.
\begin{proof}[Of Theorem~\ref{ThmSRk}]

\CHANGED{Recall from Lemma~\ref{moplemma} that the probability of an improvement when flipping $m$ bits with replacement is}
$$P(\Imp_m\mid \textsc{LO}(x)=i)=m\cdot\frac{1}{n}\cdot\left(\frac{n-i-1}{n}\right)^{m-1}+O\left(\frac{1}{n^2}\right).$$
\CHANGED{The Simple Random mechanism will choose an operator at random in each iteration. Since $|H|=k=\Theta(1)$, the probability that the randomly chosen operator finds an improvement in one iteration is}
\begin{align*}
\CHANGED{\frac{1}{k}\cdot\sum_{m=1}^k P(\Imp_m\mid \textsc{LO}(x)=i)}&=\frac{1}{k}\cdot\sum_{m=1}^k \left(m\cdot\frac{1}{n}\cdot\left(\frac{n-i-1}{n}\right)^{m-1}+O\left(\frac{1}{n^2}\right)\right)\\
&=\frac{1}{kn}+\frac{1}{k}\cdot\sum_{m=2}^k \left(m\cdot\frac{1}{n}\cdot\left(\frac{n-i-1}{n}\right)^{m-1}+O\left(\frac{1}{n^2}\right)\right)\\
&=\Omega\left(\frac{1}{n}\right).
\end{align*}
Hence, using the same nomenclature as in Theorem \ref{BDNTheorem}, the expected time for an improvement is
$A_{i}=\frac{1}{\frac{1}{k}\cdot\sum_{m=1}^k m\cdot\frac{1}{n}\cdot\left(\frac{n-i-1}{n}\right)^{m-1}}-o(n)$, \CHANGED{where the lower order $O\left(\frac{1}{n^2}\right)$ terms from $P(\Imp_m\mid \textsc{LO}(x)=i)$ are collected in the lower order $-o(n)$ term. Thus, we can sum up the expectations as stated in Theorem~\ref{BDNTheorem} to prove the theorem statement:}
\begin{align*}
\frac{1}{2}\sum_{i=0}^{n-1}A_{i}&=\frac{1}{2}\cdot\sum_{i=0}^{n-1}\frac{1}{\frac{1}{k}\cdot\sum_{m=1}^k m\cdot\frac{1}{n}\cdot\left(\frac{n-i-1}{n}\right)^{m-1}}-o(n^2)\\
&=\frac{k}{2}\cdot\sum_{i=0}^{n-1}\frac{1}{\sum_{m=1}^k m\cdot\frac{1}{n}\cdot\left(\frac{n-i-1}{n}\right)^{m-1}}- o(n^2).
\end{align*}
\end{proof}

We now prove the same deteriorating performance for the other simple \CHANGED{HHs}.

\begin{corollary}\label{CorPGRGk}
The expected runtime of the Permutation, Greedy and Random Gradient mechanisms using \CHANGED{$H=\{1\textsc{BitFlip},\dots,k\textsc{BitFlip}\}$ and $k=\Theta(1)$ for \textsc{LeadingOnes}} is 
$$\frac{k}{2}\cdot\sum_{i=0}^{n-1}\frac{1}{\sum_{m=1}^k m\cdot\frac{1}{n}\cdot\left(\frac{n-i-1}{n}\right)^{m-1}}\pm o(n^2).$$
\end{corollary}
\begin{proof}

\CHANGED{Given a bit-string $x$ with $\textsc{LO}(x)=i$,} let $p_i$ be the probability that \CHANGED{at least} one fitness improvement is constructed within $k$ fitness function evaluations. For the Greedy and Permutation mechanisms we have,
\begin{align*}
p_i&=P(\Imp_1+(1-P(\Imp_1))\cdot P(\Imp_2)+\dots+\left(\prod_{j=1}^{k-1} (1-P(\Imp_j))\right)\cdot \CHANGED{P(\Imp_k)}\\
&=\sum_{m=1}^k\left({\left(\prod_{j=1}^{m-1}\left(1-j\cdot\frac{1}{n}\cdot\left(\frac{n-i-1}{n}\right)^{j-1}\right)\right)}\cdot m\cdot\frac{1}{n}\cdot\left(\frac{n-i-1}{n}\right)^{m-1}\right)\\
&=\sum_{m=1}^k\left((1-o(1))\cdot m\cdot\frac{1}{n}\cdot\left(\frac{n-i-1}{n}\right)^{m-1}\right)\CHANGED{=\sum_{m=1}^k\left((1-o(1))\cdot P(\Imp_m)\right)},
\end{align*}
\CHANGED{where we denoted $P(\Imp_m\mid \textsc{LO}(x)=i)$ as $P(\Imp_m)$ for brevity.}

\begin{sloppypar}
To move from the second equation to the third, we \CHANGED{used that $P(\Imp_m)= O\left(\frac{1}{n}\right)$ for all $1\leq m\leq k=\Theta(1)$} and that  
${\prod_{j=1}^{m-1}(1-P(\Imp_j))=\left(1-O\left(\frac{1}{n}\right)\right)^{m}=1-o(1)}$. To \CHANGED{derive an} upper bound on the expected optimisation time we note that the difference between $\frac{k}{p_i}$ (i.e., the expected waiting time for an improvement to be
constructed), and the $A_{i}$ expected waiting times
used to prove Theorem~\ref{ThmSRk} is at most constant \CHANGED{(for $k=\Theta(1)$)}. Thus, the
difference between the expected runtimes of these mechanisms and Simple Random is limited 
to lower order terms. With probability at most $\frac{k(k-1)}{n^2}$ at least two 
mutations of the $k$ considered ones are improvements (two parallel successes of the Greedy mechanism, or two mutations performed sequentially by the Permutation mechanism). Since this occurs at most a
sublinear number of times in expectation, and the maximum expected waiting time for any
improving step is $O(n)$, the lower bound differs from the upper bound by at
most an $o(n^2)$ term. Thus, the expected runtime
of these mechanisms is also ${\frac{k}{2}\sum_{i=0}^{n-1}\frac{1}{\sum_{m=1}^k m\cdot\frac{1}{n}\cdot\left(\frac{n-i-1}{n}\right)^{m-1}}\pm o(n^2).}$
\end{sloppypar}

For the Random Gradient mechanism, we note that the probability that an operator \CHANGED{that is applied} following a success is successful again is
at most $\frac{k}{n}$ (an upper bound on the success probability of the \CHANGED{$k\textsc{BitFlip}$} operator by Lemma~\ref{moplemma}). Since there are at most $n$ improvements to be made throughout
the search space, the expected number of \CHANGED{repetitions} which produce an improvement
is at most $k= \Theta(1)$.
If the chosen operator is not successful, the Random Gradient mechanism behaves
identically to the Simple Random mechanism. Its expected runtime is therefore at least
the expected runtime of the Simple Random mechanism less an $o(n^2)$ term, and at most the expected runtime
of the Simple Random mechanism plus $n$ because there are at most $n$ iterations
where the mechanisms differ in operator selection. Thus, its expected runtime
is also $\frac{k}{2}\sum_{i=0}^{n-1}\frac{1}{\sum_{m=1}^k m\cdot\frac{1}{n}\cdot\left(\frac{n-i-1}{n}\right)^{m-1}}\pm o(n^2).$
\end{proof}

\subsection{The Generalised Random Gradient Hyper-heuristic Has the Best Possible Performance Achievable}
In this subsection, we present a rigorous theoretical analysis of the expected runtime of \CHANGED{GRG using $k=\Theta(1)$ operators for \textsc{LeadingOnes}}. The following general theorem \CHANGED{provides an upper bound on} the expected runtime of \CHANGED{GRG} using $k=\Theta(1)$ low-level stochastic mutation heuristics of different neighbourhood size for any value of $\tau$ smaller than $\frac{1}{k}n\ln n$. The theorem allows us to identify values of $\tau$ for which the expected runtime of \CHANGED{GRG} is the optimal expected runtime that may be achieved by using $k$ operators. This result will be highlighted in Corollary~\ref{corkoptbest} for values of $\tau=\omega(n)$. \CHANGED{The main result of this section has been presented in Theorem~\ref{besttheorem}, which shows that increasing the number of operators (i.e., $|H|$) that GRG has access to leads to faster expected runtimes and, in particular, expected runtimes that are \CHANGED{strictly smaller than that of} any unbiased (1+1) black box algorithm using any strict subset of the same set of operators.}

\begin{figure}[t]
\centering
	\begin{tikzpicture}
		\begin{axis}[
			xlabel=$\tau$ ($/n$), xmin=0, ymin=0.35, ymax=0.6,
			ylabel=Runtime $(/n^2)$,
			scaled ticks=false,tick label style={/pgf/number format/fixed},
			mark size=2pt,ymajorgrids=true,
			y tick label style={
				/pgf/number format/.cd,fixed,fixed zerofill,precision=2,
				/tikz/.cd
			},
			width=0.98\linewidth,
			height=2.3in,
			legend pos=north east,legend columns=1,xmax=50
		]
		\addplot[blue,mark=o]         file {Experiments/Fig1-3/GRG2OP.dat};
		\addplot[red,mark=*]          file {Experiments/Fig1-3/GRG3OP.dat};
		\addplot[green,mark=square]        file {Experiments/Fig1-3/GRG4OP.dat};
		\addplot[orange,mark=diamond]        file {Experiments/Fig1-3/GRG5OP.dat};
		
		\draw [thick, draw=brown]   (axis cs: 0,0.5) -- (axis cs: 50,0.5) -- node[above,align=right] {RLS} ++(-25,0);
		
		\legend{$k=2$, $k=3$, $k=4$, $k=5$}
		\end{axis}
	\end{tikzpicture}
	\caption{The leading constants in the theoretical upper bounds on the \CHANGED{expected} number of fitness function evaluations required by the $k$-operator Generalised Random Gradient hyper-heuristic to find the \textsc{LeadingOnes} optimum (we have used a value of $w=100,\!000$ in Theorem~\ref{kopGRG}).}
	\label{pic:upperbounds2}
\end{figure}
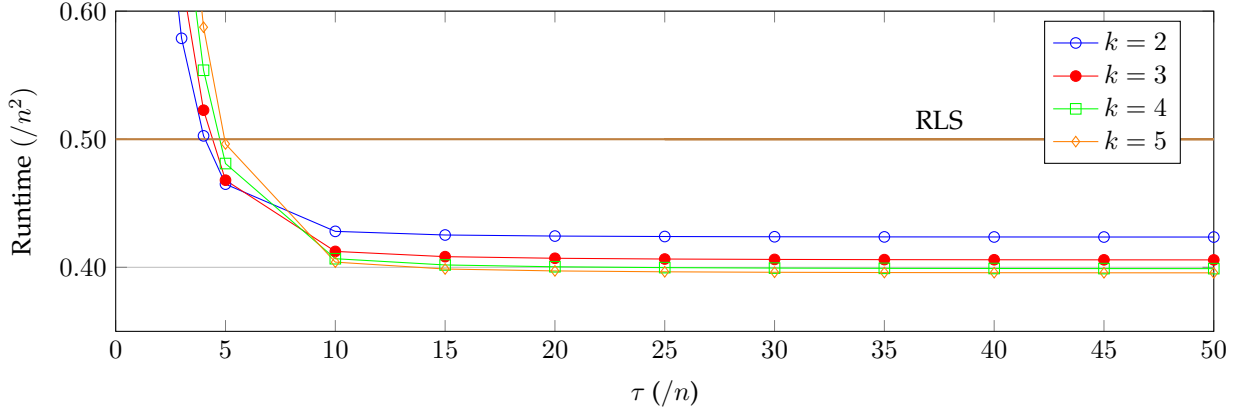

Similarly to Theorem~\ref{thm:LO-GenRandomGradient}, the proof technique partitions the search space into $w$ \CHANGED{stages}, each representing an equal range of fitness values. \CHANGED{As before, the expected runtime of the algorithm does not depend on $w$. However, $w$ does affect the upper bound we obtain. The greater the number of \CHANGED{stages}, the tighter the upper bound the theorem provides as long as $w=o\left(\frac{n}{\exp(k\tau/n)}\right)$.}

\begin{theorem}\label{kopGRG}
\sloppy \CHANGED{The expected runtime of the Generalised Random Gradient hyper-heuristic using $H=\{1\textsc{BitFlip},\dots,k\textsc{BitFlip}\}$ and $k=\Theta(1)$, for \textsc{LeadingOnes}, with $\tau\leq\left(\frac{1}{k}-\varepsilon\right)n\ln(n)$, for some constant $\CHANGED{0<\varepsilon<\frac{1}{k}}$, is at most}
$$\frac{n^2}{2}\cdot\sum_{j=1}^w \left(\frac{\left(k\cdot\frac{\tau}{n}+\left[\sum_{m=1}^k e^{m\frac{\tau}{n}\left(1-\frac{j}{w}\right)^{m-1}}\cdot M_m\left(j,w\right)\right]\right)}{w\cdot\left(\left[\sum_{m=1}^ke^{m\frac{\tau}{n}\left(1-\frac{j}{w}\right)^{m-1}}\right]-k\right)}\right)\pm o(n^2)$$
where for $m\geq2$, \[ M_m(j,w) :=
\begin{cases}
\frac{\tau}{n} & \quad \text{if } j=w \text{ or }\frac{1}{m(1-j/w)^{m-1}}>\frac{\tau}{n}\\
\frac{1}{m(1-j/w)^{m-1}} & \quad \text{otherwise}\\
\end{cases}
\]
and $M_1(j,w)=\min\!\left\{\frac{\tau}{n},1\right\}$, with \CHANGED{$w\in\omega\left(\frac{\tau}{n}\right)\cap o\left(\frac{n}{\exp(k\tau/n)}\right)$}.
\end{theorem}

Theorem~\ref{kopGRG} presents \CHANGED{upper bounds on} the expected runtime of \CHANGED{GRG}, with access to \CHANGED{an arbitrary number} $k=\Theta(1)$ of low-level stochastic mutation heuristics of different neighbourhood sizes for \textsc{LeadingOnes} \CHANGED{using any learning period} $\tau\leq\left(\frac{1}{k}-\varepsilon\right)n\ln(n)$, for some constant $\CHANGED{0<\varepsilon<\frac{1}{k}}$. \CHANGED{The bound on the expected runtime improves as $\tau$ increases up the the limit provided by the theorem.} \CHANGED{Figure \ref{pic:upperbounds2} shows the relationship between the duration of the learning period and theoretical upper bounds for different $k$-operator variants of GRG from Theorem \ref{kopGRG}. The upper bounds found by GRG with more operators are better than the ones with fewer operators, as implied by Corollary~\ref{koptbestcase}. In particular, the upper bounds for GRG with $k$-operators outperforms the best possible expected runtime for any \CHANGED{algorithm that has fewer} than $k$ operators, as implied by Theorem~\ref{besttheorem}. We depicted this result explicitly for $k=3$ and $k=5$ in Figure \ref{pic:upperbounds2v2}.} \CHANGED{When $\tau=\omega(n)$ and $\tau\leq\left(\frac{1}{k}-\varepsilon\right)$ for some constant $0<\varepsilon<\frac{1}{k}$, GRG is able to find the best possible runtime achievable with the low-level heuristics, up to lower order terms, as will be presented in Corollary~\ref{corkoptbest}.}

\CHANGED{\cite{DoerrWagner2018} calculated that the best expected runtime for any unbiased (1+1) black box algorithm using such mutation operators for \textsc{LeadingOnes} is (up to lower order terms) $\approx0.388n^2$. Corollary~\ref{corkoptbest} states that GRG can match this theoretical performance limit up to one decimal place with 4 low-level heuristics, up to two decimal places with 11 low-level heuristics and up to three decimal places with 18 low-level heuristics. In Table~\ref{Table:TauExamples} we present some of the most interesting parameter combinations of $k$ and $\tau$.}

\begin{table}[t!]
\centering
\renewcommand{\arraystretch}{1.5}
\begin{tabular}{l||cccc}
$k$ & $\tau=5n$ & $\tau=50n$ & $\tau=100n$ & $\tau=\frac{1}{10}n\ln(n)$  \\
\hline\hline
2         & $0.46493$ & $0.42363$      & $0.42329$ & $0.42329$    \\\hline
3         & $0.46802$ & $0.40579$      & $0.40525$ & $0.40525$    \\\hline
4         & $0.48102$ & $0.39897$      & $0.39830$ & $0.39830$    \\\hline
5         & $0.49630$ & $0.39568$      & $0.39492$ & $0.39492$    \\\hline\hline
11 & $3.090\times10^{23}$ & $8785.8$ & $0.38987$ & $0.38987$\\\hline\hline
18 & $1.886\times10^{44}$ & $5.363\times10^{24}$ & $1034.8$ & $0.38899$
\end{tabular}
\caption{\label{Table:TauExamples}The upper bounds in the leading constants found from various parameter combinations of the number of operators $k$ and the learning period $\tau$ in Theorem~\ref{kopGRG} with $w=100,\!000$. \CHANGED{The values reported in the fourth column are the best possible for each $k$. As the number of operators increases, larger values of $\tau$ are required for good performance ($\tau=\omega(n)$ is necessary to achieve optimal performance). Note that while Theorem~\ref{kopGRG} allows us to show optimal leading constants for $\tau=\omega(n)$, the upper bounds it provides for large values of $k$ and small values of $\tau$ are far from tight due to the generality of the theorem.}}
\end{table}

\begin{proof}[Of Theorem~\ref{kopGRG}]

We follow the proof idea for \CHANGED{GRG with} two operators (Theorem \ref{thm:LO-GenRandomGradient}). That is, we partition the optimisation process into $w$ \CHANGED{stages} based on the
value of the \textsc{LeadingOnes} fitness function: during \CHANGED{stage} $j$, the
\textsc{LeadingOnes} value of the current solution is at least $\frac{(j-1)n}{w}$ and less than
$\frac{j \cdot n}{w}$. After all $w$ \CHANGED{stages} have been completed, the global optimum with a \textsc{LeadingOnes} value of $n$ will have been found.

We \CHANGED{derive an} upper bound on the expectation of the runtime $T$ by summing up the expected values of $T_j$, the expected time
between the first random operator choice in each stage of the optimisation process and the moment when the necessary fitness value to enter the following stage is reached (that is, the time before the first random operator choice occurs in the next \CHANGED{stage}). As our analysis of $E(T_j)$ requires the \CHANGED{stage} to start with a random
choice of the mutation operator, we bound $E(T) \leq \sum_{j=1}^{w} \left(E(T_j) + E(S_{j+1}) \right)$
where $S_{j+1}$ denotes the expected number of iterations
between the first solution in \CHANGED{stage} $j+1$ \CHANGED{that is} constructed and the first random operator choice occurring in that \CHANGED{stage}. We will later show that with proper
parameter choices, the contribution of the $S_{j+1}$ terms can be bounded by
$o(n^2)$, therefore they do not affect the leading constant in the overall bound.

Let us now consider $T_j$, the number of iterations the \CHANGED{HH} spends in \CHANGED{stage} $j$. Recall that
a mutation operator is selected uniformly at random and is allowed to run until it fails to produce an improvement within $\tau$ consecutive iterations. Let $N_j$ be a random variable denoting the number
of random operator choices that the \CHANGED{HH} performs during \CHANGED{stage} $j$, and let $X_{j,1}, \ldots,
X_{j,N_j}$ be the number of iterations where the \CHANGED{HH} runs each chosen operator for. We note that $T_j = \sum_{k=1}^{N_j} X_{j,k}$ is a sum of a number
of non-negative variables, and we will later show that $E(N_j)$ is bounded. Thus,
$E(T_j)$ can be bounded by applying Wald's equation (Theorem~\ref{WaldEq}):
\begin{equation} \label{keq:LO-GRD-Tj}
	E(T_j) = E\left(\textstyle{\sum_{k=1}^{N_j}} X_{j,k}\right) \leq E(N_j) E(X_{j}),
\end{equation}
where $E(X_{j})$ denotes an upper bound
on all $E(X_{j,k})$ in \CHANGED{stage} $j$.

To bound $E(N_j)$, the expected number of times the random operator selection is performed during \CHANGED{stage} $j$, we
\CHANGED{will provide a} lower bound on the expected number of improvements found following the operator selection and apply the
Additive Drift Theorem \CHANGED{(Theorem~\ref{ADTheorem})} to find the expected number of random operator selections that occur before a
sufficient number of improvements \CHANGED{for entering} the next \CHANGED{stage} have been found.

\CHANGED{We will again denote $P(\Imp_m\mid \textsc{LO}(x)=i)$ as $P(\Imp_m)$ throughout this proof, for brevity.} Recall from Lemma~\ref{moplemma} that the improvement probability for a \CHANGED{$m\textsc{BitFlip}$} operator \CHANGED{at $\textsc{LO}(x)=i$ (with $m=\Theta(1)$)} is
$$P(\Imp_m)=m\cdot\frac{1}{n}\cdot\left(\frac{n-i-1}{n}\right)^{m-1}+O\left(\frac{1}{n^2}\right).$$

Since either the $m\cdot\frac{1}{n}\cdot\left(\frac{n-i-1}{n}\right)^{m-1}$ terms will asymptotically dominate \CHANGED{the lower order term} or the  $P(\Imp_1)=\frac{1}{n}$ term will asymptotically dominate \CHANGED{(due to the presence of the $1\textsc{BitFlip}$ operator within the set $H$)}, the lower order \CHANGED{$O\left(\frac{1}{n^2}\right)$} terms will be insignificant. In particular, we can combine the \CHANGED{contributions} of the $O\left(\frac{1}{n^2}\right)$ terms into \CHANGED{an $o(n^2)$ term} in the final bound on the expected runtime and omit them from the $P(\Imp_m)$ terms in the calculations.

Let $F_m$ denote the event that the \CHANGED{$m\textsc{BitFlip}$} operator fails to find an improvement during $\tau$ iterations. This will occur with probability $P(F_m)=\left(1-P(\Imp_m)\right)^\tau$. Within \CHANGED{stage} $j$, the ancestor individual has at most $i=\frac{jn}{w}-1$ and at least $i=\frac{(j-1)n}{w}$ leading 1-bits. Hence,
\begin{align*}
P(F_m)&\leq \left(1-m\cdot\frac{1}{n}\left(\frac{n-\left(\frac{jn}{w}-1\right)-1}{n}\right)^{m-1}\right)^{\tau}\leq\left(1-m\cdot\frac{1}{n}\left(1-\frac{j}{w}\right)^{m-1}\right)^{\tau}\\
&\leq e^{-\frac{m\tau}{n}\left(1-\frac{j}{w}\right)^{m-1}},
\end{align*}
and
\begin{align*}
P(F_m)&\geq \left(1-m\cdot\frac{1}{n}\left(\frac{n-\left(\frac{(j-1)n}{w}\right)-1}{n}\right)^{m-1}\right)^{\tau}\geq\left(1-m\cdot\frac{1}{n}\left(1-\frac{j-1}{w}\right)^{m-1}\right)^{\tau}\\
&\geq e^{-\frac{m\tau}{n}\left(1-\frac{j-1}{w}\right)^{m-1}}-\frac{1}{n}.
\end{align*}

A geometric distribution with parameter $p = P(F_m)$ can be used to model the number
of improvements that the \CHANGED{$m\textsc{BitFlip}$} operator finds prior to failing to find
an improvement for $\tau$ iterations. The expectation of this distribution is $\frac{1-p}{p} = \frac{1}{p} - 1$. We can model these events with a geometric random variable since GRG immediately restarts a period of $\tau$ iterations with the chosen heuristic when a successful mutation occurs (i.e., it sets $c_t=0$ \CHANGED{in Algorithm~\ref{alg:GRG}}). Hence, there are no periods of $\tau$ consecutive unsuccessful mutations \CHANGED{apart from the period when} the operator finally fails.
Combined over all $k$ operators, the expected number of improvements $D_j$
produced following a single random operator selection during \CHANGED{stage} $j$ is greater than:
$$E(D_j)\geq \sum_{m=1}^k \frac{1}{k}\left(\frac{1}{P(F_m)}-1\right)\geq \frac{1}{k}\left(\sum_{m=1}^k e^{\frac{m\tau}{n}\left(1-\frac{j}{w}\right)^{m-1}}\right)-1,$$
\CHANGED{where the last inequality is obtained by inserting} the upper bounds for $P(F_m)$. We use this expectation as the drift in the Additive Drift Theorem to provide an upper bound for $E(N_j)$. Recall that for each \CHANGED{stage} at most $\frac{n}{w}$ fitness values \CHANGED{have to be crossed}. Since bits beyond the leading \CHANGED{1-bits} prefix and the first zero bit
remain uniformly distributed, in expectation $\frac{n}{2w}$ improvements by mutation are required. If each step of a random process contributes $E(D_j)$ expected improvements by mutation, the expected number of steps required to complete \CHANGED{stage} $j$ is at most
\begin{align} \label{keq:grd-stage-drift}
	E(N_j) \leq \frac{\frac{n}{2w}}{E(D_j)} \leq \frac{k}{2}\cdot\frac{n}{w}\cdot\frac{1}{\left(\sum_{m=1}^k e^{\frac{m\tau}{n}\left(1-\frac{j}{w}\right)^{m-1}}\right)-k},
\end{align}
by the Additive Drift Theorem.

\NEWCHANGED{We apply Wald's equation to bound $E(X_{j})$, the expected number of iterations before a selected mutation operator fails to produce an improvement for $\tau$ iterations: let $S_m$ be the number of improvements found by the operator before it fails, and $W_{m,1}, W_{m,2}, \ldots, W_{m,S_m}$ be the number of iterations it took to find each of those improvements. Then, once selected, the $m\textsc{BitFlip}$
operator runs for
$$E(X_{j} \mid m\textsc{BitFlip}) = \tau + \sum_{s=1}^{S_m} E(W_{m,s}) = E(S_m)E(W_{m,1} \mid S_m \geq 1),$$
where $\tau$ accounts for the iterations immediately before the failure and the sum counts the expected number of iterations
preceding each constructed improvement. Recall that $E(S_m) = \frac{1}{P(F_m)} - 1$ by the properties of the
geometric distribution and observe that $E(W_{m,1} \mid S_m \geq 1) = E(W_{m,1} \mid W_{m,1} \leq \tau) \leq
\min\{\tau, E(W_{m,1})\}$. Using a waiting time argument (i.e., considering the expectation of a geometric random variable with success probability  $p=P(\Imp_m)$) for $E(W_{m,1})$, we get,}
\begin{align*}
	E(X_{j}\mid \CHANGED{m\textsc{BitFlip}}) &\leq \tau + \left(\frac{1}{P(F_m)} - 1\right)\cdot\min\!\left\{\tau, \frac{1}{m\cdot\frac{1}{n}\cdot\left(\frac{n-\left(\frac{jn}{w}-1\right)-1}{n}\right)^{m-1}}\right\}\\
	&= \tau+\left(\frac{1}{e^{-\frac{m\tau}{n}\left(1-\frac{j-1}{w}\right)^{m-1}}-\frac{1}{n}}-1\right)\cdot\min\!\left\{\tau,\frac{n}{m\left(1-\frac{j}{w}\right)^{m-1}}\right\}\\
	&\leq \tau+e^{\frac{m\tau}{n}\left(1-\frac{j}{w}\right)^{m-1}}\cdot\min\!\left\{\tau,\frac{n}{m\left(1-\frac{j}{w}\right)^{m-1}}\right\}+O(1),
\end{align*}
\CHANGED{with $e^{\frac{m\tau}{n}\left(1-\frac{j-1}{w}\right)^{m-1}}=(1+o(1))\cdot e^{\frac{m\tau}{n}\left(1-\frac{j}{w}\right)^{m-1}}$ for $m=\Theta(1)$ and $w=\omega\left(\frac{\tau}{n}\right)$.} Combining these conditional expectations yields,
\begin{align}\label{keq:grd-op-phase-length}
\nonumber E(X_j)&\leq\sum_{m=1}^k \frac{1}{k}\cdot E(X_{j}\mid \NEWCHANGED{m\textsc{BitFlip}})\\
\nonumber &\leq \tau+\sum_{m=1}^k \left[\frac{1}{k}\cdot e^{\frac{m\tau}{n}\left(1-\frac{j}{w}\right)^{m-1}}\cdot\min\!\left\{\tau,\frac{n}{m\left(1-\frac{j}{w}\right)^{m-1}}\right\} \right]+O(k)\\
\nonumber &\leq n\left( \frac{\tau}{n}+\sum_{m=1}^k \left[\frac{1}{k}\cdot e^{\frac{m\tau}{n}\left(1-\frac{j}{w}\right)^{m-1}}\cdot\min\!\left\{\frac{\tau}{n},\frac{1}{m\left(1-\frac{j}{w}\right)^{m-1}}\right\} \right]\right)+O(k)\\
\nonumber &\leq \frac{n}{k}\left( \frac{k\tau}{n}+\sum_{m=1}^k \left[e^{\frac{m\tau}{n}\left(1-\frac{j}{w}\right)^{m-1}}\cdot\min\!\left\{\frac{\tau}{n},\frac{1}{m\left(1-\frac{j}{w}\right)^{m-1}}\right\} \right]\right)+O(k).\\
\end{align}

At $j=w$, we have that $\min\!\left\{\frac{\tau}{n},\frac{1}{m\left(1-\frac{j}{w}\right)^{m-1}}\right\}$ is undefined. However, we know that the chosen operator would only be applied for a maximum of $\tau$ steps. Hence, we define $M_1(j,w)=\min\!\left\{\frac{\tau}{n},1\right\}$ and for $m\geq2$, \[ M_m(j,w) :=
\begin{cases}
\frac{\tau}{n} & \quad \text{if } j=w \text{ or }\frac{1}{m\left(1-\frac{j}{w}\right)^{m-1}}>\frac{\tau}{n}\\
\frac{1}{m\left(1-\frac{j}{w}\right)^{m-1}} & \quad \text{otherwise.}\\
\end{cases}
\]
Finally, we return to bounding the overall expected runtime of the
\CHANGED{HH}. Recall that $S_j$ denotes the number of iterations that the algorithm spends in \CHANGED{stage} $j$ while using the random operator chosen during \CHANGED{stage} $j-1$ and since $E(S_j) \leq \max\left(E(X_{j} \mid \CHANGED{m\textsc{BitFlip}}\right)_{m=1,\dots, k} < k E(X_{j})= O\left(n\cdot\exp\left(\frac{k\tau}{n}\right)\right)$ and $w=o\left(\frac{n}{\exp\left(\frac{k\tau}{n}\right)}\right)$, \CHANGED{it follows that} $\sum_{j=1}^{w} E(S_{j+1}) = o(n^2)$.

Substituting the bounds provided in Eq.~\ref{keq:grd-stage-drift} and Eq.~\ref{keq:grd-op-phase-length} into Eq.~\ref{keq:LO-GRD-Tj} yields the theorem statement:
\begin{align*}
	E(T) & \leq \sum_{j=1}^{w} \left(E(T_j) + E(S_{j+1}) \right)
	\leq \left(\sum_{j=1}^{w} E(N_j) E(X_{j})\right) + o(n^2) \\
	& \leq o(n^2) + \frac{n^2}{2} \times \sum_{j=1}^{w} \frac{1}{w}\cdot \frac{\frac{k\tau}{n}+\sum_{m=1}^k \left[e^{\frac{m\tau}{n}\left(1-\frac{j}{w}\right)^{m-1}}\cdot M_m\left(j,w\right) \right]}{\left(\sum_{m=1}^k e^{\frac{m\tau}{n}\left(1-\frac{j}{w}\right)^{m-1}}\right)-k} .
\end{align*}
\end{proof}

We now show that for appropriate values of the parameter $\tau$, the $k$-operator \CHANGED{GRG runs in} the best possible expected runtime \CHANGED{achievable} for a mechanism using $k$-operators, up to lower order terms. The following corollary provides an upper bound on the expected runtime for sufficiently large learning periods (i.e., $\tau=\omega(n)$). \CHANGED{The bound matches the best expected runtime achievable proven in Theorem~\ref{koptbestcase}.}
\begin{corollary}\label{corkoptbest}
\sloppy \CHANGED{The expected runtime of the Generalised Random Gradient hyper-heuristic using $H=\{1\textsc{BitFlip},\dots,k\textsc{BitFlip}\}$ and $k=\Theta(1)$ for \textsc{LeadingOnes}, with $\tau$ that satisfies both $\tau=\omega(n)$ and $\tau\leq\left(\frac{1}{k}-\varepsilon\right)n\ln(n)$, for some constant $\CHANGED{0<\varepsilon<\frac{1}{k}}$, is at most}
$$\frac{1}{2}\left(\sum_{i=0}^{\frac{n}{k}-1}\frac{1}{k\cdot\frac{1}{n}\cdot\left(\frac{n-i-1}{n}\right)^{k-1}}+\sum_{m=1}^{k-1}\sum_{i=\frac{n}{m+1}}^{\frac{n}{m}-1}\frac{1}{m\cdot\frac{1}{n}\cdot\left(\frac{n-i-1}{n}\right)^{m-1}}\right)\pm o(n^2).$$
\end{corollary}
\begin{proof}
This proof follows a similar structure to \CHANGED{the proof} of Corollary~\ref{CorGRGOpt}, \CHANGED{i.e., we aim to simplify the sum from Theorem~\ref{kopGRG}. 

Theorem~\ref{kopGRG} states that,}
$$E(T)\leq\frac{n^2}{2}\cdot\sum_{j=1}^w \left(\frac{\left(k\cdot\frac{\tau}{n}+\left[\sum_{m=1}^k e^{m\frac{\tau}{n}\left(1-\frac{j}{w}\right)^{m-1}}\cdot M_m\left(j,w\right)\right]\right)}{w\cdot\left(\left[\sum_{m=1}^ke^{m\frac{\tau}{n}\left(1-\frac{j}{w}\right)^{m-1}}\right]-k\right)}\right)\pm o(n^2).$$
Since $\tau=\omega(n)$, we have that $\frac{\tau}{n}=\omega(1)$. Hence, we can simplify the $M_m(j,w)$ terms. Recall that we have $M_1(j,w)=\min\!\left\{\frac{\tau}{n},1\right\}=1$, and for $m\geq2$,\[ M_m(j,w) :=
\begin{cases}
\frac{\tau}{n} & \quad \text{if } j=w \text{ or }\frac{1}{m\left(1-\frac{j}{w}\right)^{m-1}}>\frac{\tau}{n}\\
\frac{1}{m\left(1-\frac{j}{w}\right)^{m-1}} & \quad \text{otherwise.}\\
\end{cases}
\]
Since $\frac{\tau}{n}=\omega(1)$, we have that $\frac{\tau}{n}>\frac{1}{m\left(1-\frac{j}{w}\right)^{m-1}}$. Hence, we can simplify \CHANGED{$M_m(j,w)$ as follows}:
\[ M_m(j,w) :=
\begin{cases}
\frac{\tau}{n} & \quad \text{if } j=w\\
\frac{1}{m\left(1-\frac{j}{w}\right)^{m-1}} & \quad \text{otherwise.}\\
\end{cases}
\]
In particular, for $j\neq w$, we have \CHANGED{that $M_m(j,w)=\Theta(1)$, for $m\geq1$}. 

We can further simplify the sum from Theorem~\ref{kopGRG} by noting that, since $\frac{\tau}{n}=\omega(1)$, the exponential term in each summand will be asymptotically dominant. In particular, the $\frac{k\tau}{n}$ term in the numerator and the $-k$ term in the denominator will be asymptotically dominated. Hence, we can relegate these terms into the $o(n^2)$ lower order term.

With these modifications we can simplify the sum \CHANGED{as follows,}
\begin{align*}
E(T)&\leq\frac{n^2}{2}\cdot\sum_{j=1}^w \left(\frac{\left(k\cdot\frac{\tau}{n}+\left[\sum_{m=1}^k e^{m\frac{\tau}{n}\left(1-\frac{j}{w}\right)^{m-1}}\cdot M_m\left(j,w\right)\right]\right)}{w\cdot\left(\left[\sum_{m=1}^ke^{m\frac{\tau}{n}\left(1-\frac{j}{w}\right)^{m-1}}\right]-k\right)}\right)\pm o(n^2)\\
&\leq\frac{n^2}{2}\cdot\sum_{j=1}^w \left(\frac{\left(\sum_{m=1}^k e^{m\frac{\tau}{n}\left(1-\frac{j}{w}\right)^{m-1}}\cdot M_m\left(j,w\right)\right)}{w\cdot\left(\left[\sum_{m=1}^ke^{m\frac{\tau}{n}\left(1-\frac{j}{w}\right)^{m-1}}\right]\right)}\right)\pm o(n^2).
\end{align*}

To further simplify the sum, we consider when each of the $k$ exponential summand terms will be asymptotically dominant in the numerator. Since $\frac{\tau}{n}=\omega(1)$, whichever exponential term has the largest exponent will asymptotically dominate the other terms. In particular, the $m_{th}$ term will dominate when $\frac{m\tau}{n}\left(1-\frac{j}{w}\right)^{m-1}$ is the largest amongst $1\leq m\leq k$. Comparing the terms at $m$ and $m+1$, we see that the $m_{th}$ term dominates the $(m+1)_{th}$ term when $j\geq\frac{w}{m+1}$. Continuing these calculations, we get that the $m_{th}$ term will dominate all the others (i.e., it will have the largest exponent) when $\frac{w}{m+1}\leq j\leq \frac{w}{m}-1$ for $m<k$, and the $k_{th}$ term will dominate when $j\leq \frac{w}{k}-1$. If a term is asymptotically dominated, we can relegate it into the $o(n^2)$ lower order term. Hence, the only terms remaining in the numerator in the period when the $m_{th}$ operator dominates will be the $m_{th}$ exponential term multiplied by $M_m(j,w)$. The denominator will similarly only contain the $m_{th}$ exponential term, multiplied by $w$. The sum will hence simplify to a sum of $M_m(j,w)$ terms:
\begin{align*}
E(T)&\leq\frac{n^2}{2}\cdot\left(\sum_{j=1}^{\frac{w}{k}}M_k(j,w)+\sum_{m=1}^{k-1}\sum_{j=\frac{w}{m+1}+1}^{\frac{w}{m}}M_m(j,w)\right)\pm o(n^2).\\
&=\frac{n^2}{2}\cdot\left(\sum_{j=1}^{\frac{w}{k}}\frac{1}{k\left(1-\frac{j}{w}\right)^{k-1}}+\sum_{m=1}^{k-1}\sum_{j=\frac{w}{m+1}+1}^{\frac{w}{m}}\frac{1}{m\left(1-\frac{j}{w}\right)^{m-1}}\right)\pm o(n^2).\\
&=\frac{n}{2}\cdot\left(\sum_{j=1}^{\frac{w}{k}}\frac{1}{k\cdot\frac{1}{n}\cdot\left(1-\frac{j}{w}\right)^{k-1}}+\sum_{m=1}^{k-1}\sum_{j=\frac{w}{m+1}+1}^{\frac{w}{m}}\frac{1}{m\cdot\frac{1}{n}\cdot\left(1-\frac{j}{w}\right)^{m-1}}\right)\pm o(n^2).
\end{align*}

Recall that \CHANGED{stage} $j$ refers to when $\frac{(j-1)n}{w}\leq i\leq\frac{jn}{w}-1$ where $i$ is the LO value of the current solution. In particular, if we substitute $i\geq \frac{(j-1)n}{w}$, (or $j\leq \frac{iw}{n}+1$) into each summand, we get,
$$\frac{1}{m\cdot\frac{1}{n}\cdot\left(1-\frac{j}{w}\right)^{m-1}}\leq\frac{1}{m\cdot\frac{1}{n}\cdot\left(\frac{n-i-1}{n}\right)^{m-1}}.$$
Furthermore, using the upper and lower bounds on $i$ to adjust the bounds on the sums (which will cancel out the multiplicative $n$ term) gives the final result:
\begin{align*}
E(T)&\leq\frac{n}{2}\cdot\left(\sum_{j=1}^{\frac{w}{k}}\frac{1}{k\cdot\frac{1}{n}\cdot\left(1-\frac{j}{w}\right)^{k-1}}+\sum_{m=1}^{k-1}\sum_{j=\frac{w}{m+1}+1}^{\frac{w}{m}}\frac{1}{m\cdot\frac{1}{n}\cdot\left(1-\frac{j}{w}\right)^{m-1}}\right)\pm o(n^2)\\
&\leq\frac{1}{2}\left(\sum_{i=0}^{\frac{n}{k}-1}\frac{1}{k\cdot\frac{1}{n}\cdot\left(\frac{n-i-1}{n}\right)^{k-1}}+\sum_{m=1}^{k-1}\sum_{i=\frac{n}{m+1}}^{\frac{n}{m}-1}\frac{1}{m\cdot\frac{1}{n}\cdot\left(\frac{n-i-1}{n}\right)^{m-1}}\right)\pm o(n^2).
\end{align*}
\end{proof}

We can now prove the main result of this section, i.e., Theorem~\ref{besttheorem}, \CHANGED{which states that the expected runtime of GRG using $H=\{1\textsc{BitFlip},\dots,k\textsc{BitFlip}\}$ and $k=\Theta(1)$ with $\tau$ that satisfies both $\tau=\omega(n)$ and $\tau\leq\left(\frac{1}{k}-\varepsilon\right)n\ln(n)$, for some constant $0<\varepsilon<\frac{1}{k}$, for \textsc{LeadingOnes} is smaller than the best-possible expected runtime for any unbiased (1+1) black box algorithm using any strict subset of $H=\{1\textsc{BitFlip},\dots,k\textsc{BitFlip}\}$.}
\begin{proof}[Of Theorem~\ref{besttheorem}]

Corollary~\ref{corkoptbest} shows that GRG with $k$ mutation operators can match the best-possible peformance \CHANGED{of any unbiased (1+1) black box} algorithm with $k$ mutation operators \CHANGED{for \textsc{LeadingOnes}} from Theorem~\ref{koptbestcase}, up to lower order terms. The results from Corollary~\ref{corkoptbest} and Theorem~\ref{koptbestcase} imply that the only difference in the best case expected runtime for GRG with $k$ operators (\CHANGED{i.e., using $H=\{1\textsc{BitFlip},\dots,k\textsc{BitFlip}\}$}) and the best-possible expected runtime for an algorithm with $k-1$ operators (i.e., using $\CHANGED{\{1\textsc{BitFlip},\dots,(k-1)\textsc{BitFlip}\}}$) occurs when $0\leq i\leq \frac{n}{k}-1$. In this region, the best case expected \NEWCHANGED{runtime} of the $k$-operator \CHANGED{GRG} matches the expected \NEWCHANGED{runtime of an algorithm using only} the \CHANGED{$k\textsc{BitFlip}$} operator, up to lower order terms, while the expected \NEWCHANGED{runtime} of the best-possible $k-1$-operator algorithm matches the expected \NEWCHANGED{runtime of an algorithm using only} the $(k-1)\textsc{BitFlip}$ operator, \NEWCHANGED{up to lower order terms}.

We know from Lemma~\ref{moplemma2} that the expected runtime to optimise this area when applying the \CHANGED{$k\textsc{BitFlip}$} operator will be smaller than the expected runtime when applying the \CHANGED{$k-1$\textsc{BitFlip}} operator, and thus using the \CHANGED{$k\textsc{BitFlip}$} operator will be faster. Since in the rest of the search space, the best case $k$-operator GRG and the best-possible $k-1$-operator algorithm will have the same expected performance (up to lower order terms), the best case expected runtime of GRG with $k$ operators will be \CHANGED{smaller} than the best-possible expected runtime of any algorithm with $k-1$-operators.

\CHANGED{Reiterating} the argument for the best-possible performance of an algorithm with $\CHANGED{(1,\dots,m)\textsc{BitFlip}}$ operators ($m<k-1$) \CHANGED{allows us to conclude} that the best case expected runtime of the $k$-operator \CHANGED{GRG} is faster than the best-possible performance of any algorithm with $m<k$ operators. In particular, GRG with \CHANGED{ $H=\{1\textsc{BitFlip},\dots,k\textsc{BitFlip}\}$} is faster than the best-possible algorithm using any strict subset of $\CHANGED{\{1\textsc{BitFlip},\dots,k\textsc{BitFlip}\}}$. \CHANGED{Lemma~\ref{moplemma2} combined with a similar argument as used above with $k-1$ operators} implies that GRG will outperform the best-possible algorithm \CHANGED{equipped with a subset of the $k$ operators} in the area of the search space where the latter algorithm is missing any of the $\CHANGED{\{1\textsc{BitFlip},\dots,k\textsc{BitFlip}\}}$ operators (i.e., if the best-possible algorithm is missing the \CHANGED{$m\textsc{BitFlip}$} operator, GRG will outperform it in the region where the \textsc{LeadingOnes} fitness value satisfies $\frac{n}{m+1}\leq \textsc{LO}(x)\leq\frac{n}{m}-1$), \CHANGED{hence its expected runtime is smaller}.
\end{proof}

\CHANGED{
\subsection{Anytime Performance}
In this subsection, we present the expected fixed target running times of the algorithms considered thus far. Fixed target running times have been previously presented by \cite{DoerrPinto2018}. In order to achieve the runtimes presented in Corollary~\ref{cor:fixedtarget} (and depicted in Figure~\ref{fig:FixedTarget}) we have adapted Theorem~\ref{ThmSRk} and Theorem~\ref{kopGRG} to sum up to a fixed target $\textsc{LO}(x)=X\leq n$, while keeping the problem size as $n$ (respectively, $w$) within the summands. 

\begin{corollary}\label{cor:fixedtarget}
Let $X\leq n$. The expected time needed to reach for the first time a search point $x$ of \textsc{LeadingOnes}-value at least $X$ is at most:
\begin{itemize}
\item $\frac{k}{2}\cdot\sum_{i=0}^{X-1}\frac{1}{\sum_{m=1}^k m\cdot\frac{1}{n}\cdot\left(\frac{n-i-1}{n}\right)^{m-1}}\pm o(n^2)$ for the Simple Random mechanism with $k$ operators;
\item \begin{sloppypar}$\frac{n^2}{2}\cdot
\left(\sum_{j=1}^{\lceil X\cdot\frac{w}{n}\rceil}\frac{\left(k\cdot\frac{\tau}{n}+\left[\sum_{m=1}^k e^{m\frac{\tau}{n}\left(1-\frac{j}{w}\right)^{m-1}}\cdot M_m\left(j,w\right)\right]\right)}{w\cdot\left(\left[\sum_{m=1}^ke^{m\frac{\tau}{n}\left(1-\frac{j}{w}\right)^{m-1}}\right]-k\right)}\right) + o(n^2)$ for GRG using ${H=\{1\textsc{BitFlip},\dots,k\textsc{BitFlip}\}}$, with $\tau\leq\left(\frac{1}{k}-\varepsilon\right)n\ln(n)$ for some constant $0<\varepsilon<\frac{1}{k}$, where $M_m(j,w)$ are as defined in Theorem~\ref{kopGRG}.\end{sloppypar}
\end{itemize}
\end{corollary}

\begin{figure}[t]
\centering
	\begin{tikzpicture}
		\begin{axis}[
			xlabel=Target $\textsc{LO}(x)$ value, xmin=0, xmax=10000,ymin=0, ymax=70,
			ylabel=Expected First Hitting Time (Millions),
			scaled ticks=false,tick label style={/pgf/number format/fixed},
			mark size=2pt,ymajorgrids=true,
			y tick label style={
				/pgf/number format/.cd,fixed,
				/tikz/.cd
			},
			width=0.97\linewidth,
			height=2.3in,
			legend pos=north west,legend columns=1
		]
		\pgfplotstableread{FixedTarget/FixedTargetTheory.dat} \FixedTargetTheory
		\addplot[black,densely dotted,thick]   table [x index=0,y index=1] \FixedTargetTheory;
		\addplot[red,dotted,thick]  table [x index=0,y index=2] \FixedTargetTheory;
		\addplot[black!60!green,dashed,thick]   table [x index=0,y index=3] \FixedTargetTheory;
		\addplot[red,densely dashed,thick]   table [x index=0,y index=4] \FixedTargetTheory;
		\addplot[black!60!green,thick]   table [x index=0,y index=5] \FixedTargetTheory;
		\addplot[cyan,loosely dashed,thick]   table [x index=0,y index=6] \FixedTargetTheory;
			
		\legend{RLS: $0.5$, Simple$_2$: $0.54931$, Simple$_3$: $0.65288$, GRG$_2$: $0.42815$, GRG$_3$: $0.41202$, GRG$_4$: $0.40632$}
		\end{axis}
	\end{tikzpicture}
	\caption{\CHANGED{Expected fixed target running times for \textsc{LeadingOnes}, with dimension $n=10,\!000$. The values shown in the legend for Simple Random and GRG variants with $\tau=10n$ are the normalised (by $n^2$) expected optimisation times.}}
	\label{fig:FixedTarget}
\end{figure}
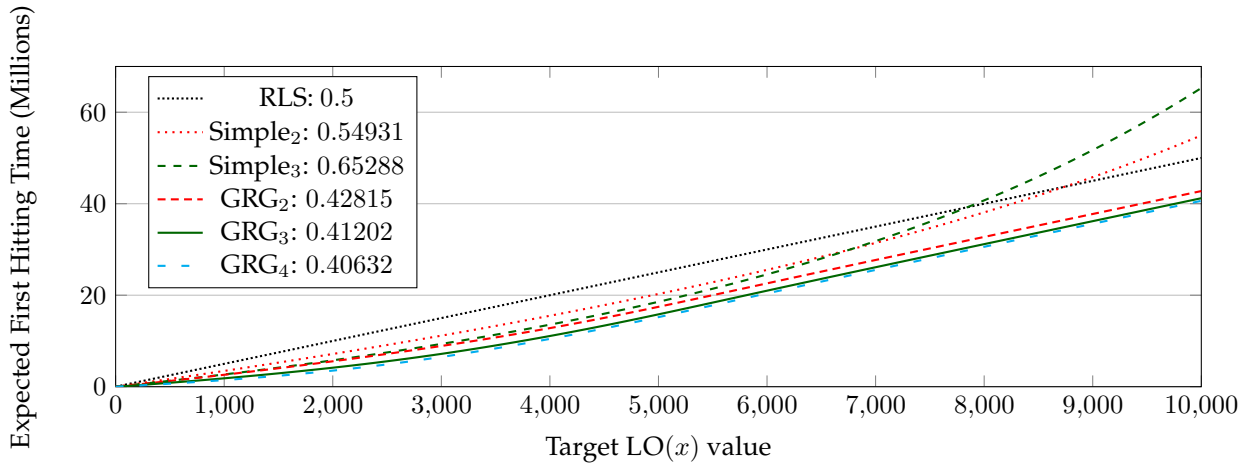

Figure~\ref{fig:FixedTarget} shows the expected fixed target running times using the set $H=\{1\textsc{BitFlip},\dots,k\textsc{BitFlip}\}$ for the Simple Random mechanism with $k=1, 2, 3$ and GRG with $\tau=10n$ and $k=2, 3, 4$ for \textsc{LeadingOnes}. We have used $w=1,\!000$ in Theorem~\ref{kopGRG} to achieve the GRG plots.  

We see that all the GRG variants outperform the respective Simple Random variants for any fixed target independent of how many operators are used by GRG. All GRG variants are close to matching the anytime performance of the optimal GRG and the greater the number of operators the HH has access to, the better the anytime performance. For example, GRG with 4 operators \NEWCHANGED{outperforms} all other algorithms for every fixed target $0<X\leq n$.

Interestingly, the Simple Random mechanisms with more operators outperform those with less operators for smaller targets. For example, for a fixed target of $\textsc{LO}(x)=\frac{n}{2}=5,\!000$, the expected first hitting time of Simple Random with 2 operators is $\approx18.9\%$ smaller than that of RLS (i.e., Simple Random with 1 operator, Simple$_1$), while Simple Random with 3 operators is $\approx26.0\%$ faster than RLS.

}

\CHANGED{
\subsection{Extension of the Results to Standard Randomised Local Search Heuristics}
In this subsection, we will extend the results presented so far in Section~\ref{sec:MoreThanTwo} such that they also hold for hyper-heuristics that select from the set $H=\{\textsc{RLS}_1,\dots,\textsc{RLS}_k\}$. In particular, we will show that the main result of this section (i.e., Theorem~\ref{besttheorem}) also holds for RLS algorithms with different neighbourhood sizes (that flip distinct bits without replacement). For this purpose, we will show that the difference in the improvement probabilities of the $\textsc{RLS}_m$ algorithms and the $m\textsc{BitFlip}$ operators (from Lemma~\ref{moplemma}) are limited to lower order $O\left(\frac{1}{n^2}\right)$ terms. A simple argument will then prove that the remaining results of this section also hold for $\textsc{RLS}_m$ operators.

\begin{lemma}\label{rlsbitflip}
Given a bit-string $x$ with $LO(x)=i\leq n-m$, the probability of improvement of an $\textsc{RLS}_m$ algorithm that flips $m=\Theta(1)$ bits without replacement is
$$P(\Imp_{\textsc{RLS}_m}\mid \textsc{LO}(x)=i)=m\cdot\frac{1}{n}\cdot\left(\frac{n-i-1}{n}\right)^{m-1}\pm O\left(\frac{1}{n^2}\right).$$
\end{lemma}
\begin{proof}
To prove the statement we will show that the differences between the $P(\Imp_m\mid \textsc{LO}(x)=i)$ from Lemma~\ref{moplemma} and $P(\Imp_{\textsc{RLS}_m}\mid \textsc{LO}(x)=i)$ for each $m=\Theta(1)$ are limited to lower order $O\left(\frac{1}{n^2}\right)$ terms.

\cite{DoerrWagner2018} showed that the improvement probability of $\textsc{RLS}_m$ at the state $\textsc{LO}(x)=i$ is 
$$P(\Imp_{\textsc{RLS}_m}\mid \textsc{LO}(x)=i)=\frac{\binom{n-i-1}{m-1}}{\binom{n}{m}}.$$
Note that the improvement probability of $\textsc{RLS}_m$ for $\textsc{LO}(x)=i>n-m$ is $0$. 

A simple calculation shows that the largest difference between the improvement probability of $\textsc{RLS}_m$ and the improvement probability of the $m\textsc{BitFlip}$ mutation operator (from Lemma~\ref{moplemma}) occurs when $\textsc{LO}(x)=i=0$. We now calculate this difference (dropping the state dependence for brevity):
\begin{align*}
P(\Imp_{\textsc{RLS}_m})-P(\Imp_m)&\leq\frac{\binom{n-1}{m-1}}{\binom{n}{m}}-\left(m\cdot\frac{1}{n}\cdot\left(\frac{n-1}{n}\right)^{m-1}+O\left(\frac{1}{n^2}\right)\right)\\
&=\frac{m}{n}-\left(\frac{m}{n}\cdot\left(\frac{n-1}{n}\right)^{m-1}+O\left(\frac{1}{n^2}\right)\right)\\
&=\frac{m}{n}\left(1-\left(1-\frac{1}{n}\right)^{m-1}\right)-O\left(\frac{1}{n^2}\right)\\
&=\frac{m}{n}\left(1-\left(1-\frac{m-1}{n}\pm o\left(\frac{1}{n}\right)\right)\right)-O\left(\frac{1}{n^2}\right)\\
&=\frac{m(m-1)}{n^2}\pm O\left(\frac{1}{n^2}\right)=O\left(\frac{1}{n^2}\right),
\end{align*}
since $m=\Theta(1)$. Hence, the difference in the improvement probabilities of the two mutation operators is limited to lower order $O\left(\frac{1}{n^2}\right)$ terms and the lemma statement holds.
\end{proof}

Through a simple application of Lemma~\ref{rlsbitflip} (rather than applying Lemma~\ref{moplemma}), the results for the simple mechanisms presented in Subsection~\ref{Section:SimpleMore} (i.e., Theorem~\ref{ThmSRk} and Corollary~\ref{CorPGRGk}), and the general results for GRG (i.e., Theorem~\ref{kopGRG} and Corollary~\ref{corkoptbest}), also hold for the hyper-heuristics using the heuristic set $H=\{\textsc{RLS}_1,\dots,\textsc{RLS}_k\}$, up to lower order $\pm o(n^2)$ terms.

We can also extend Theorem~\ref{koptbestcase} and Theorem~\ref{besttheorem} to algorithms using the heuristic set $\{\textsc{RLS}_1,\dots,\textsc{RLS}_k\}$. We know from Theorem~\ref{koptbestcase} that the $m\textsc{BitFlip}$ operator is optimal (i.e., it has the highest probability of improvement amongst all \textsc{BitFlip} operators) during the time when $\frac{n}{m+1}\leq i\leq \frac{n}{m}-1$. \cite{Doerr2018Arxiv} showed that $\textsc{RLS}_m$ is optimal (i.e., has the highest probability of improvement among all \textsc{RLS} operators) when $\frac{n-m}{m+1}\leq i\leq\frac{n}{m}-\frac{m-1}{m}$. The differences in the expected runtime of the best-possible algorithms using the two heuristic sets would therefore amount to the expected time spent improving from (at most) $m+1=\Theta(1)$ different \textsc{LeadingOnes} fitness values. Since the expected time to improve from each of these fitness values is at most $\Theta(n)$, the difference between the expected runtime of the best-possible algorithm using $\{\textsc{RLS}_1,\dots,\textsc{RLS}_k\}$ operators and that using $H=\{1\textsc{BitFlip},\dots,k\textsc{BitFlip}\}$ operators, for \textsc{LeadingOnes} is therefore limited to lower order $o(n^2)$ terms. Hence, Theorem~\ref{koptbestcase} also holds for the heuristic set $\{\textsc{RLS}_1,\dots,\textsc{RLS}_k\}$, up to lower order $\pm o(n^2)$ terms. With similar arguments, Theorem~\ref{besttheorem} also holds for sets of \textsc{RLS} operators.
}

\section{Complementary Experimental Analysis}\label{sec:experiments}
In the previous sections we proved that \CHANGED{GRG} performs efficiently for the \textsc{LeadingOnes} benchmark function for large enough problem sizes $n$. In this section we present some experimental results to shed light on its performance for different problem sizes up to $n=10^8$. All parameter combinations have been simulated $10,\!000$ times.

\begin{figure}[t]
\centering
	\begin{tikzpicture}
		\begin{axis}[
			xlabel=$\tau$ ($/(n\ln(n))$), xmin=0, ymin=0.4, ymax=0.6,
			ylabel=Runtime $(/n^2)$,
			scaled ticks=false,tick label style={/pgf/number format/fixed},
			mark size=2pt, xmax=3, ymajorgrids=true,
			legend style={at={(0.98,0.05)},anchor=south east},
			yscale=1,
			y tick label style={
				/pgf/number format/.cd,fixed,fixed zerofill,precision=3,
				/tikz/.cd
			},
			ytick={0.4,0.423,0.45,0.5,0.55,0.6},
			width=1\linewidth,
			height=2.3in
		]
		\addplot[red] file {Experiments/Fig4/nIncreasingTau-1.dat};
		\addplot[red,dashed] file {Experiments/Fig4/nIncreasingTau-50k-1.dat};
		\draw [thick, draw=brown]   (axis cs: 0,0.5) -- (axis cs: 3,0.5) -- node[above,align=right] {RLS} ++(-30,0);
		\draw [thick, draw=teal]   (axis cs: 0,.4232867952) -- (axis cs: 3,.4232867952) -- node[below,align=left] {$2_\mathrm{Opt}$} ++(-500,0);
		\legend{{$n=10,\!000$}, {$n=50,!000$}}
		\end{axis}
	\end{tikzpicture}
\caption{Average number of fitness function evaluations required by the hyper-heuristics with $k=2$ operators to find the \textsc{LeadingOnes} optimum \CHANGED{in relation to the duration of the learning period} for $n=10,\!000$ (solid), $n=50,\!000$ (dashed).}
\label{Fig:TauIncrease}
\end{figure}
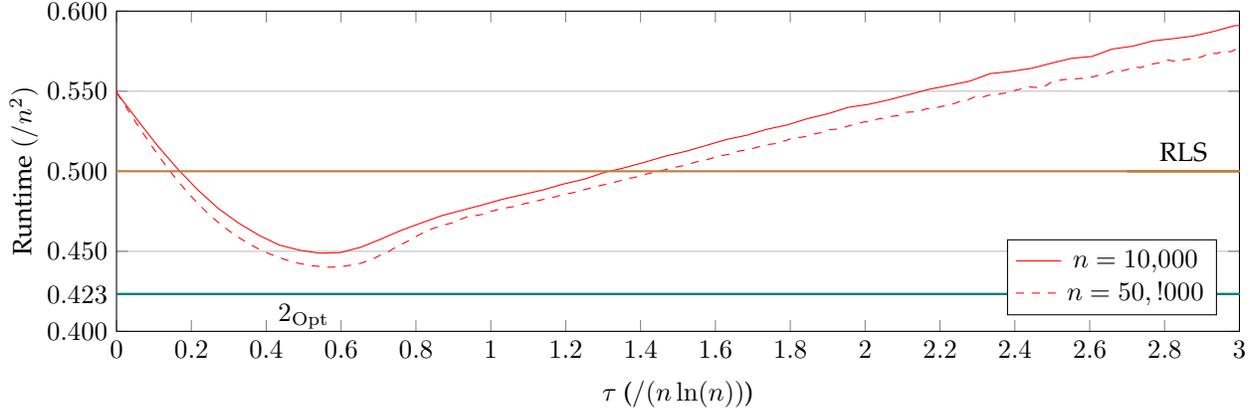

In order to efficiently handle larger problem dimensions experimentally, we do not simulate each individual mutation performed by the \CHANGED{HH}, but rather sample the waiting times for a fitness-improving mutation to occur using a geometric distribution (with the success probability $p$ depending on the current operator and \textsc{LeadingOnes} value of the current solution). Specifically, suppose $T_\mathrm{im}$ is a random variable denoting the number of mutations required to get one fitness-improving mutation. Since $T_\mathrm{im}$ counts the number of independent trials each with probability $p$ of success up to and including the first successful trial, $P(T_\mathrm{im} \leq k) = 1-(1-p)^k$ by the properties of the geometric distribution. Given access to a $\mathrm{uniform\,}(0,1)$-distributed random variable $U$, $T_\mathrm{im}$ can be sampled by computing $\left\lceil\frac{ \log(1-U)}{\log(1-p)} \right\rceil$.

\subsection{Two Low-level Heuristics ($k=2$)}
We first consider \CHANGED{GRG} using two operators only (i.e., \CHANGED{$H=\{1\textsc{BitFlip},\textsc{2BitFlip}\}$}) and look at the impact \CHANGED{on the average runtime} of the parameter $\tau$ and of the problem size $n$.

Figure \ref{Fig:TauIncrease} shows \CHANGED{how the average} runtimes of GRG for \textsc{LeadingOnes} \CHANGED{vary with the duration of the learning period $\tau$} for \CHANGED{problem sizes} $n=10,\!000$ and $n=50,\!000$. \CHANGED{The performance of GRG} clearly depends on the choice of $\tau$. It is worth noting that as the problem size increases, for $\tau\approx0.55n\ln(n)$, the runtime seems to be approaching the optimal performance proven in Corollary~\ref{CorGRGOpt} (i.e., $\frac{1+\ln(2)}{4}n^2\approx0.42329n^2$). For well chosen $\tau$ values, the \CHANGED{HH} beats the expected runtime of RLS and also the experimental runtime for the recently presented \CHANGED{reinforcement learning HH that chooses between different neighbourhood sizes for $\textsc{RLS}_k$ \citep{DoerrEtAl2016}. They reported an average runtime of $0.450n^2$ for the parameter choices they used}. As $\tau$ increases past $0.6n^2$, we see a detriment in the performance of the \CHANGED{HH}. It is worth noting, however, that for $n=50,\!000$, it is required that $\tau>1.5n\ln(n)=811,\!483$ to be worse than the RLS expected runtime of $0.5n^2$, indicating that the parameter is robust.

\begin{figure}[t]
\centering
	\begin{tikzpicture}
		\begin{semilogxaxis}[
			xlabel=$n$, xmin=100, ymin=0.4, ymax=0.55, xmax=1.1*10^8,
			ylabel=Runtime $(/n^2)$,
			legend style={at={(0.98,0.95)},anchor=north east},
			scaled ticks=false,tick label style={/pgf/number format/fixed},
			mark size=2pt,ymajorgrids=true,
			yscale=1,
			y tick label style={
				/pgf/number format/.cd,fixed,fixed zerofill,precision=3,
				/tikz/.cd
			},
			ytick={0.4,0.423,0.45,0.5,0.55},
			width=0.98\linewidth,
			height=2.3in
		]
		\addplot[blue,dashed] file {Experiments/Fig5/RG-a1-0.25nln.dat};
		\addplot[red,mark=x] file {Experiments/Fig5/RG-a1-0.5nln.dat};
		\addplot[black,mark=o] file {Experiments/Fig5/RG-a1-0.6nln.dat};
		\addplot[black!60!green,mark=diamond] file {Experiments/Fig5/RG-a1-0.75nln.dat};
		\draw [thick, draw=teal]   (axis cs: 0,.4232867952) -- (axis cs: 1.1*10^8,.4232867952) -- node[below] {$2_\mathrm{Opt}$} ++(-5,0);
		\legend{$\tau=0.25n\ln n$, $\tau=0.5n\ln n$, $\tau=0.6n\ln n$, $\tau=0.75n\ln n$}
		\end{semilogxaxis}
	\end{tikzpicture}
\caption{{Average number of fitness function evaluations required for the Generalised Random Gradient hyper-heuristic with $k=2$ operators to find the \textsc{LeadingOnes} optimum as the problem size $n$ increases.
}}
\label{Fig:nIncrease}
\end{figure}
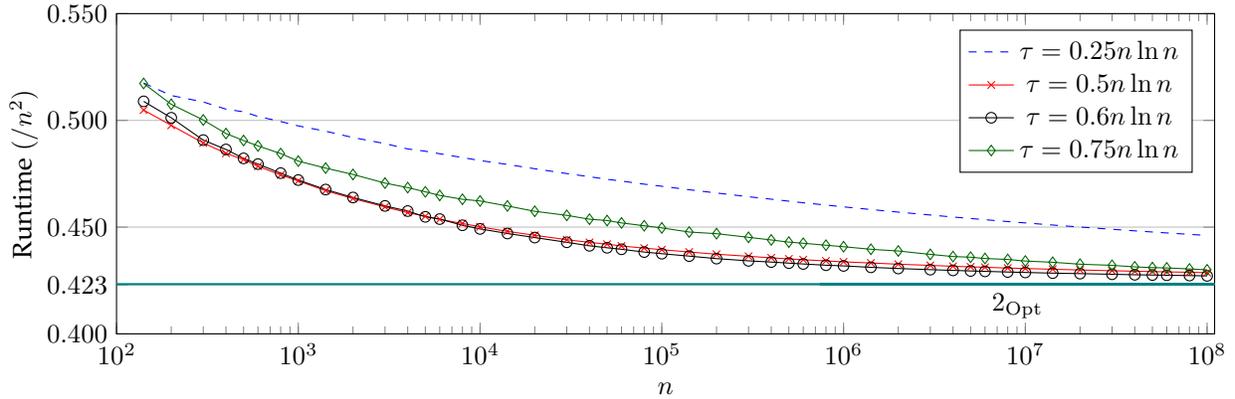
Figure \ref{Fig:nIncrease} shows the effects of increasing the problem size $n$ for a variety of fixed values of $\tau$. We can see that an increased problem size leads to faster average runtimes. In particular, the \CHANGED{HH} requires a problem size of at least 200 before it outperforms RLS. The performance difference between the $\tau$ values decreases with increased $n$, indicating that further increasing $n$ would lead to similar, optimal performance for a large range of values, as implied by Corollary~\ref{CorGRGOpt}. For $n=10^8$, the runtime for $\tau=0.6n\ln n$ is $\approx0.42716n^2$, only slightly deviated from the optimal value of $\approx0.42329n^2$.

\subsection{More Than Two Low-level Heuristics ($k\geq 2$)}
We now consider \CHANGED{GRG} using $k$ operators, $\CHANGED{H=\{1\textsc{BitFlip},\dots,k\textsc{BitFlip}\}}$, and look at the impact of the parameter $\tau$.

\CHANGED{

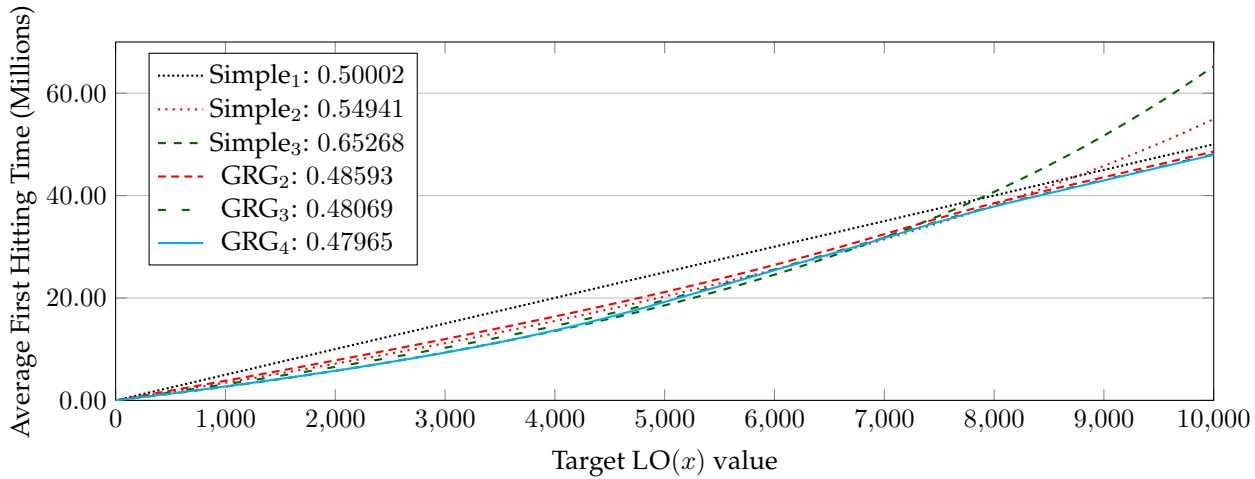
\begin{figure}[t]
\centering
	\begin{tikzpicture}
		\begin{axis}[
			xlabel=Target $\textsc{LO}(x)$ value, xmin=0, xmax=10000,ymin=0, ymax=70,
			ylabel= Average First Hitting Time (Millions),
			scaled ticks=false,tick label style={/pgf/number format/fixed},
			mark size=2pt,ymajorgrids=true,
			y tick label style={
				/pgf/number format/.cd,fixed,fixed zerofill,precision=2,
				/tikz/.cd
			},
			width=0.98\linewidth,
			height=2.5in,
			legend pos=north west,legend columns=1
		]
		\pgfplotstableread{FixedTarget/Experiment.dat} \FixedTargetExp
		\addplot[black,densely dotted,thick]   table [x index=0,y index=1] \FixedTargetExp;
		\addplot[red,dotted,thick]   table [x index=0,y index=2] \FixedTargetExp;
		\addplot[black!60!green,dashed,thick]   table [x index=0,y index=3] \FixedTargetExp;
		\addplot[red,densely dashed,thick]   table [x index=0,y index=4] \FixedTargetExp;
		\addplot[black!60!green,loosely dashed,thick]   table [x index=0,y index=5] \FixedTargetExp;
		\addplot[cyan,thick]   table [x index=0,y index=6] \FixedTargetExp;
			
		\legend{Simple$_1$: $0.50002$, Simple$_2$: $0.54941$, Simple$_3$: $0.65268$, GRG$_2$: $0.48593$, GRG$_3$: $0.48069$, GRG$_4$: $0.47965$}
		\end{axis}
	\end{tikzpicture}
	\caption{\CHANGED{Experimental fixed target running times for \textsc{LeadingOnes}, with dimension $n=10,\!000$. The values shown in the legend for Simple Random and GRG variants with $\tau=10n$ are the normalised (by $n^2$) average optimisation times.}}
	\label{fig:ExpFixedTarget}
\end{figure}

Figure~\ref{fig:ExpFixedTarget} shows the average fixed target running times using the set $H=\{1\textsc{BitFlip},\dots,k\textsc{BitFlip}\}$ for the Simple Random mechanism with $k=1, 2, 3$ operators, and GRG with $\tau=10n$ and $k=2, 3, 4$ for \textsc{LeadingOnes}. We have used a problem size of $n=10,\!000$. The conclusions drawn from the experiments match the theoretical results shown in Figure~\ref{fig:FixedTarget}. We see that although the final average runtime of the Simple Random mechanisms with more operators increases, they are faster initially. In particular, while $\textsc{LO}(x)\leq\frac{n}{2}=5,000$,  Simple Random with 3 operators is preferable to the variant with 2 or 1 (i.e., RLS). This Simple Random variant even outperforms GRG (with our arbitrarily chosen learning period $\tau=10n$) for fixed targets smaller than $\approx7,000$. For GRG we see, just like in  Figure~\ref{fig:FixedTarget}, that GRG with more operators is more effective than GRG with fewer operators for any fixed target. However, the difference is smaller than the theoretical results suggest, implying that while more operators are preferable, GRG with fewer operators is still effective for these problem sizes i.e., as the problem size increases so does the difference in performance in favour of larget sets $H$.

}

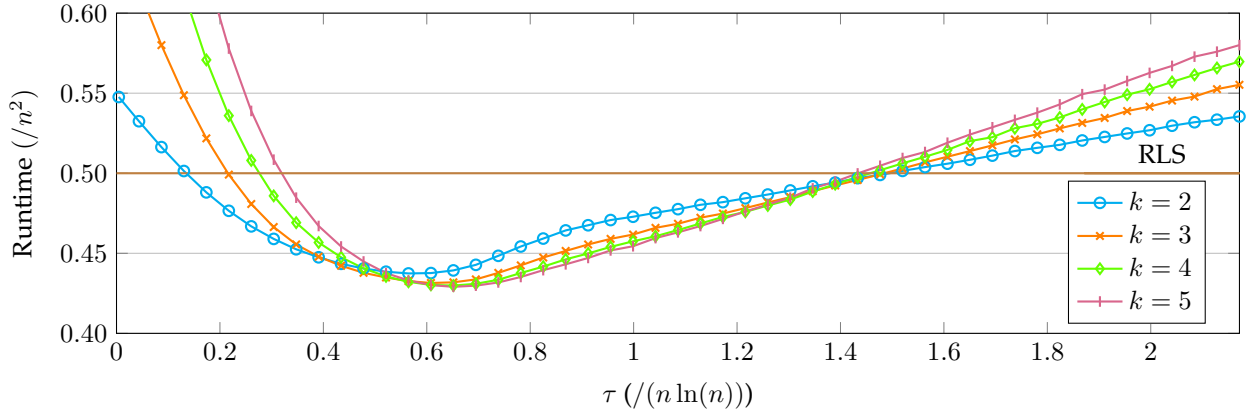
\begin{figure}[t]
	\begin{tikzpicture}
		\begin{axis}[
			xlabel=$\tau$ ($/(n\ln(n))$), xmin=0, ymin=0.4, ymax=0.6,
			ylabel=Runtime $(/n^2)$,
			scaled ticks=false,tick label style={/pgf/number format/fixed},
			mark size=2pt, xmax=2.1715, ymajorgrids=true,
			legend pos=south east,
			y tick label style={
				/pgf/number format/.cd,fixed,fixed zerofill,precision=2,
				/tikz/.cd
			},
			width=1\linewidth,
			height=2.3in
		]
		\addplot[thick,cyan,mark=o]          file {Experiments/Fig6/GRG-100k-2.dat};
		\addplot[thick,orange,mark=x]          file {Experiments/Fig6/GRG-100k-3.dat};
		\addplot[thick,green!60!yellow,mark=diamond]         file {Experiments/Fig6/GRG-100k-4.dat};
		\addplot[thick,purple!60!white,mark=|]         file {Experiments/Fig6/GRG-100k-5.dat};
		\draw [thick, draw=brown]   (axis cs: 0,0.5) -- (axis cs: 2.1715,0.5) -- node[above,align=right] {RLS} ++(-30,0);
		\legend{$k=2$, $k=3$, $k=4$, $k=5$}
		\end{axis}
	\end{tikzpicture}
\caption{Average number of fitness function evaluations required by the Generalised Random Gradient hyper-heuristic with $k$ operators to find the \textsc{LeadingOnes} optimum, $n=100,\!000$.}
\label{pic:GRGk2}
\end{figure}
Figure~\ref{pic:GRGk2} shows the average runtimes for \CHANGED{GRG} when it chooses between 2 to 5 mutation operators \CHANGED{with different neighbourhood sizes} for a \textsc{LeadingOnes} problem size $n=100,\!000$. We see that incorporating more operators can be beneficial to the performance of GRG. Whilst the 2-operator \CHANGED{HH} achieves a best average performance of $\approx0.43742n^2$, the 5-operator \CHANGED{HH} achieves a best average performance of $\approx0.42906n^2$. It is, however, important to set $\tau$ \CHANGED{appropriately} to achieve the best performance. For $0\leq0.35n\ln(n)\leq\tau$ and $\tau\geq1.5n\ln(n)$, the 2-operator \CHANGED{HH} outperforms the others. The results imply that the more operators are incorporated in the \CHANGED{HH}, the shorter the range of $\tau$ values for which it performs the best in comparison with the \CHANGED{HH} with fewer operators. We know from Corollary~\ref{corkoptbest} that for a sufficiently large problem size, GRG with 5-operators will tend towards the theoretical optimal \CHANGED{expected runtime} of $\approx0.39492n^2$. Furthermore, we know from Theorem~\ref{besttheorem} that GRG equipped with more operators will \CHANGED{be faster}.

Recently \cite{DoerrWagner2018} analysed a (1+1) EA for \textsc{LeadingOnes}, where the mutation rate is updated on-the-fly. The number of bits to be flipped is sampled from a binomial distribution $\operatorname{B}(n,p)$ \CHANGED{and, differently from standard bit mutation, they do not allow a \textsc{0BitFlip} to occur by resampling until a non-zero bitflip occurs}\footnote{Note that while not allowing 0-bit flips is a reasonable choice for an extremely simplified EA such as the (1+1) EA, it is doubtful that it is a good idea in general. For instance, we conjecture that the expected runtime of the ($\mu+1$) EA \citep{Witt2006} and ($\mu+1$) GA \citep{CorusOliveto2018,CorusOliveto2019} would deteriorate for \textsc{OneMax} and \textsc{LeadingOnes}, while non-elitist EAs and GAs would require exponential time to optimise any function with a unique optimum \citep{CorusEtAl2018,OlivetoWitt2014,OlivetoWitt2015}. An alternative implementation that allows 0-bit flips to occur, but simply avoids evaluating offspring that are identical to their parents, would solve the issue while at the same time producing the results reported by \cite{DoerrWagner2018}. Such an idea was suggested by \cite{DoerrPinto2018}.}. Then a multiplicative comparison-based update rule, similar to the 1/5th rule from combinatorial optimisation \CHANGED{\citep{KernEtAl2004}} is applied to update the parameter $p$ (i.e., the mutation rate).  A successful mutation increases the mutation rate by a multiplicative factor $A>1$, while an unsuccessful mutation multiplies the mutation rate by a multiplicative factor $b<1$. 

\CHANGED{An experimental analysis is performed to tentatively identify the leading constants in the expected runtime for the best combinations of $A$ and $b$. The best leading constant that has been identified is 0.4063 (using the best parameter configuration for which an average of at least 1000 runs has been reported; i.e., $A=1.2$ and $b=0.85$). While it is unclear whether their identified value indeed scales with the problem size (they only test problem sizes up to $n=1500$), such a leading constant is worse than the $\approx0.388n^2$ achieved by GRG. However, we note that the best possible expected runtime that such an EA that learns to automatically adapt the neighbourhood size to optimality can achieve is $\approx0.404n^2$ \citep{DoerrDoerrLengler2019Arxiv}, due to the random sampling of the number of bits to flip, i.e., the best adaptive RLS is faster than the best adaptive EA.} 

\section{Conclusion}\label{sec:Conc}
\CHANGED{Our foundational understanding of the performance of hyper-heuristics (HHs) is improving. 
Algorithm selection from algorithm portfolio systems and selection HHs generally use sophisticated machine learning algorithms to identify online (during the run) which low-level mechanisms have better performance during different stages of the optimisation process. Recently it has been proven that a reinforcement learning mechanism, which assigns different scores to each low-level heuristic \NEWCHANGED{according to how well it performs, allows a simple HH to run in the best expected runtime achievable with RLS$_k$ low-level heuristics for the \textsc{OneMax} benchmark problem, up to lower order terms}~\citep{DoerrEtAl2016}.

In this paper we considered whether sophisticated learning mechanisms are always necessary
for a HH to effectively learn to apply the right low-level heuristic.}

We considered four of the most simple learning mechanisms \CHANGED{commonly} applied in the literature to combinatorial optimisation problems, \CHANGED{namely, Simple Random, Permutation, Greedy and Random Gradient and showed that they all have the same performance for the \textsc{LeadingOnes} benchmark function when equipped with RLS$_k$ low-level heuristics. While the former three mechanisms do not attempt to learn from the past performance of the low-level heuristics, the idea behind Random Gradient is to continue applying a heuristic so long as it is successful. We argued that looking at the performance  of a heuristic after one single application is insufficient to appreciate whether it is a good choice or not. To this end,}  
%
we generalised the existing \CHANGED{Random Gradient} learning mechanism to allow success to be measured over a longer period of time 
\CHANGED{and called such time the \textit{learning period}}. \NEWCHANGED{We showed that the Generalised Random Gradient (GRG) HH can learn to adapt the neighbourhood size $k=\Theta(1)$ of $\textsc{RLS}_k$ optimally during the run  for the \textsc{LeadingOnes} benchmark function. As a byproduct, we proved that, up to lower order terms, GRG has the best possible runtime achievable by any algorithm that uses the same low level heuristics. In particular, it is faster than well-known unary unbiased evolutionary and local-search algorithms, including the (1+1) Evolutionary Algorithm ((1+1) EA) and Randomised Local Search (RLS). Furthermore, we also showed that for targets smaller than $n$ (i.e., anytime performance), the advantages of GRG over RLS and EAs using standard bit mutation are even greater (i.e., they are even better as approximation algorithms for the problem).}

To apply the generalised \CHANGED{HH}, a value for the learning period $\tau$ is required. Although our results indicate that $\tau$ is a fairly robust parameter (i.e., for $n=10,\!000$, \CHANGED{GRG} achieved faster experimental runtimes than that of the (1+1)~EA for all tested values of $\tau$ between $1$ and $10^6$, and faster than RLS for all tested values of $\tau$ between $28,\!000$ and $120,\!000$), setting it appropriately will lead to optimal performance. Clearly $\tau$ must be large enough to have at least a constant expected number of successes within $\tau$ steps, if the \CHANGED{HH} has to learn about the operator performance. Naturally, setting too large values of $\tau$ may lead to large runtimes since switching operators requires $\Omega(\tau)$ steps. 

We have also rigorously shown that the performance of the simple mechanisms deteriorates as the choice of operators \CHANGED{with different neighbourhood sizes} increases, while the performance of \CHANGED{GRG} improves with a larger choice, as desired for practical applications. \CHANGED{In particular, GRG is able to outperform in expectation any unbiased (1+1) black box algorithm with access to strictly smaller sets of operators.}

Recently, \cite{DLOW2018} have equipped the \NEWCHANGED{GRG} \CHANGED{HH} with an adaptive update rule to automatically adapt the parameter $\tau$ throughout the run (i.e., the learning period can change its duration during the optimisation process). They proved that the \CHANGED{HH} is able to achieve the same optimal performance, up to lower order terms, for the \textsc{LeadingOnes} benchmark function. This so called Adaptive Random Gradient \CHANGED{HH}, equipped with two operators, experimentally outperforms the best setting of \CHANGED{GRG}, confirming that $\tau$ should not be fixed throughout the optimisation process.

Several further directions can be explored in future work. Firstly, the performance of GRG on a broader class of problems, including classical ones from combinatorial optimisation, should be rigorously studied. Secondly, more sophisticated \CHANGED{HHs} that have shown superior performance in practical applications should be analysed, such as the \CHANGED{machine learning approaches which keep track of the historical performance of each low-level heuristic and use this information to decide which one to be applied next}. \NEWCHANGED{In particular, it should be highlighted when more sophisticated learning mechanisms are required and the reasons behind the requirement.} Thirdly, the understanding of HHs that switch between elitist and non-elitist low-level heuristics for multimodal functions should be improved~\citep{LissovoiEtAl2019}. Finally, considering more sophisticated low-level heuristics (e.g. with different population sizes) will bring a greater understanding of the \NEWCHANGED{general} performance of selection \CHANGED{HHs} \NEWCHANGED{and their wider application in real-world optimisation. Another area where a foundational theoretical understanding is lacking is that of parameter tuning \citep{HallEtAl2019Arxiv}.}

\paragraph*{Acknowledgements}
This work was supported by EPSRC under grant EP/M004252/1.

\small

\bibliographystyle{apalike}
\bibliography{references}

\begin{thebibliography}{}

\bibitem[Afshani et~al., 2013]{AfshaniEtAl2013}
Afshani, P., Agrawal, M., Doerr, B., Doerr, C., Larsen, K.~G., and Mehlhorn, K.
  (2013).
\newblock The query complexity of finding a hidden permutation.
\newblock In Brodnik, A., L{\'o}pez-Ortiz, A., Raman, V., and Viola, A.,
  editors, {\em Space-Efficient Data Structures, Streams, and Algorithms:
  Papers in Honor of J. Ian Munro on the Occasion of His 66th Birthday}, pages
  1--11. Springer.

\bibitem[Alanazi and Lehre, 2014]{AlanaziLehre2014}
Alanazi, F. and Lehre, P.~K. (2014).
\newblock Runtime analysis of selection hyper-heuristics with classical
  learning mechanisms.
\newblock In {\em IEEE Congress on Evolutionary Computation}, CEC `14, pages
  2515--2523. IEEE.

\bibitem[Alanazi and Lehre, 2016]{AlanaziLehre2016}
Alanazi, F. and Lehre, P.~K. (2016).
\newblock Limits to learning in reinforcement learning hyper-heuristics.
\newblock In {\em Evolutionary Computation in Combinatorial Optimization},
  EvoCOP `16, pages 170--185. Springer.

\bibitem[Asta and {\"O}zcan, 2014]{AstaOzcan2014}
Asta, S. and {\"O}zcan, E. (2014).
\newblock An apprenticeship learning hyper-heuristic for vehicle routing in
  {HyFlex}.
\newblock In {\em {IEEE} Symposium on Evolving and Autonomous Learning
  Systems}, EALS `14, pages 65--72. {IEEE}.

\bibitem[Berbero{\v{g}}lu and Uyar, 2011]{BerberogluUyar2011}
Berbero{\v{g}}lu, A. and Uyar, A.~{\c{S}}. (2011).
\newblock Experimental comparison of selection hyper-heuristics for the
  short-term electrical power generation scheduling problem.
\newblock In {\em Applications of Evolutionary Computation}, EvoApplications
  `11, pages 444--453. Springer.

\bibitem[Bezerra et~al., 2018]{BezerraEtAl2018}
Bezerra, L. C.~T., López-Ibáñez, M., and Stützle, T. (2018).
\newblock A large-scale experimental evaluation of high-performing multi- and
  many-objective evolutionary algorithms.
\newblock {\em Evolutionary Computation}, 26(4):621--656.

\bibitem[B{\"o}ttcher et~al., 2010]{BoettcherEtAl2010}
B{\"o}ttcher, S., Doerr, B., and Neumann, F. (2010).
\newblock Optimal fixed and adaptive mutation rates for the leadingones
  problem.
\newblock In {\em Parallel Problem Solving from Nature}, PPSN `10, pages 1--10.
  Springer.

\bibitem[Burke et~al., 2003]{BurkeEtAlJournal2003}
Burke, E., Kendall, G., and Soubeiga, E. (2003).
\newblock A tabu-search hyperheuristic for timetabling and rostering.
\newblock {\em Journal of Heuristics}, 9(6):451--470.

\bibitem[Burke et~al., 2013]{BurkeEtAl2013}
Burke, E.~K., Gendreau, M., Hyde, M., Kendall, G., Ochoa, G., {\"O}zcan, E.,
  and Qu, R. (2013).
\newblock Hyper-heuristics: A survey of the state of the art.
\newblock {\em Journal of the Operational Research Society}, 64(12):1695--1724.

\bibitem[Buzdalov and Buzdalova, 2015]{BuzdalovBuzdalova2015}
Buzdalov, M. and Buzdalova, A. (2015).
\newblock Can onemax help optimizing leadingones using the {EA}+{RL} method?
\newblock In {\em IEEE Congress on Evolutionary Computation}, CEC `15, pages
  1762--1768. IEEE.

\bibitem[Carvalho~Pinto and Doerr, 2017]{DoerrPinto2018}
Carvalho~Pinto, E. and Doerr, C. (2017).
\newblock Towards a more practice-aware runtime analysis.
\newblock In {\em Conference on Artificial Evolution}, EA `17, pages 298--305.
  Springer.

\bibitem[{Corus} et~al., 2018]{CorusEtAl2018}
{Corus}, D., {Dang}, D., {Eremeev}, A.~V., and {Lehre}, P.~K. (2018).
\newblock Level-based analysis of genetic algorithms and other search
  processes.
\newblock {\em IEEE Transactions on Evolutionary Computation}, 22(5):707--719.
\newblock To appear at GECCO `19.

\bibitem[Corus and Oliveto, 2018]{CorusOliveto2018}
Corus, D. and Oliveto, P.~S. (2018).
\newblock Standard steady state genetic algorithms can hillclimb faster than
  mutation-only evolutionary algorithms.
\newblock {\em IEEE Transactions on Evolutionary Computation}, 22(5):720--732.

\bibitem[Corus and Oliveto, 2019]{CorusOliveto2019}
Corus, D. and Oliveto, P.~S. (2019).
\newblock On the benefits of populations on the exploitation speed of standard
  steady-state genetic algorithms.
\newblock {\em CoRR}, abs/1903.10976.

\bibitem[Cowling et~al., 2001]{CowlingEtAl2000}
Cowling, P., Kendall, G., and Soubeiga, E. (2001).
\newblock A hyperheuristic approach to scheduling a sales summit.
\newblock In {\em Practice and Theory of Automated Timetabling}, PATAT `01,
  pages 176--190. Springer.

\bibitem[Cowling et~al., 2002]{CowlingEtAl2002}
Cowling, P., Kendall, G., and Soubeiga, E. (2002).
\newblock Hyperheuristics: A tool for rapid prototyping in scheduling and
  optimisation.
\newblock In {\em Applications of Evolutionary Computing}, EvoWorkshops `02,
  pages 1--10. Springer.

\bibitem[Doerr, 2018]{Doerr2018Arxiv}
Doerr, B. (2018).
\newblock Better runtime guarantees via stochastic domination.
\newblock {\em CoRR}, abs/1801.04487.

\bibitem[Doerr and Doerr, 2018]{DoerrDoerrBookChapter2018}
Doerr, B. and Doerr, C. (2018).
\newblock Theory of parameter control for discrete black-box optimization:
  Provable performance gains through dynamic parameter choices.
\newblock {\em CoRR}, abs/1804.05650.

\bibitem[Doerr et~al., 2019]{DoerrDoerrLengler2019Arxiv}
Doerr, B., Doerr, C., and Lengler, J. (2019).
\newblock Self-adjusting mutation rates with provably optimal success rules.
\newblock {\em CoRR}, abs/1902.02588.

\bibitem[Doerr et~al., 2016]{DoerrEtAl2016}
Doerr, B., Doerr, C., and Yang, J. (2016).
\newblock k-bit mutation with self-adjusting k outperforms standard bit
  mutation.
\newblock In {\em Parallel Problem Solving from Nature}, PPSN `16, pages
  824--834. Springer.

\bibitem[Doerr and Krejca, 2018]{DoerrKrejca2018}
Doerr, B. and Krejca, M.~S. (2018).
\newblock Significance-based estimation-of-distribution algorithms.
\newblock In {\em Proceedings of the Genetic and Evolutionary Computation
  Conference}, GECCO '18, pages 1483--1490. ACM.

\bibitem[Doerr and K{\"{u}}nnemann, 2013]{DoerrK13}
Doerr, B. and K{\"{u}}nnemann, M. (2013).
\newblock How the (1+{\(\lambda\)}) evolutionary algorithm optimizes linear
  functions.
\newblock In {\em Proceeedings of the Genetic and Evolutionary Computation
  Conference}, GECCO `13, pages 1589--1596. ACM.

\bibitem[Doerr et~al., 2018]{DLOW2018}
Doerr, B., Lissovoi, A., Oliveto, P.~S., and Warwicker, J.~A. (2018).
\newblock On the runtime analysis of selection hyper-heuristics with adaptive
  learning periods.
\newblock In {\em Proceedings of the Genetic and Evolutionary Computation
  Conference}, GECCO '18, pages 1015--1022. ACM.

\bibitem[Doerr and Lengler, 2018]{DoerrLengler2018}
Doerr, C. and Lengler, J. (2018).
\newblock The (1+1) elitist black-box complexity of leadingones.
\newblock {\em Algorithmica}, 80(5):1579--1603.

\bibitem[Doerr and Wagner, 2018]{DoerrWagner2018}
Doerr, C. and Wagner, M. (2018).
\newblock Simple on-the-fly parameter selection mechanisms for two classical
  discrete black-box optimization benchmark problems.
\newblock In {\em Proceedings of the Genetic and Evolutionary Computation
  Conference}, GECCO '18, pages 943--950. ACM.

\bibitem[Friedrich et~al., 2016]{FriedrichEtAl2016}
Friedrich, T., K\"{o}tzing, T., and Krejca, M.~S. (2016).
\newblock {EDA}s cannot be balanced and stable.
\newblock In {\em Proceedings of the Genetic and Evolutionary Computation
  Conference}, GECCO '16, pages 1139--1146. ACM.

\bibitem[Gibbs et~al., 2010]{GibbsEtAl2010}
Gibbs, J., Kendall, G., and {\"O}zcan, E. (2010).
\newblock Scheduling {English} football fixtures over the holiday period using
  hyper-heuristics.
\newblock In {\em Parallel Problem Solving from Nature}, PPSN `10, pages
  496--505. Springer.

\bibitem[Hall et~al., 2019]{HallEtAl2019Arxiv}
Hall, G.~T., Oliveto, P.~S., and Sudholt, D. (2019).
\newblock On the impact of the cutoff time on the performance of algorithm
  configurators.
\newblock {\em CoRR}, abs/1904.06230.
\newblock To appear at GECCO `19.

\bibitem[He et~al., 2012]{HeEtAl2012}
He, J., He, F., and Dong, H. (2012).
\newblock Pure strategy or mixed strategy?
\newblock In {\em Evolutionary Computation in Combinatorial Optimization},
  EvoCOP `12, pages 218--229. Springer.

\bibitem[He and Yao, 2001]{HeYao2001}
He, J. and Yao, X. (2001).
\newblock Drift analysis and average time complexity of evolutionary
  algorithms.
\newblock {\em Artificial Intelligence}, 127(1):57--85.

\bibitem[Jansen and Zarges, 2011]{JansenZarges2011}
Jansen, T. and Zarges, C. (2011).
\newblock Analysis of evolutionary algorithms: From computational complexity
  analysis to algorithm engineering.
\newblock In {\em Proceedings of the Workshop on Foundations of Genetic
  Algorithms}, FOGA '11, pages 1--14. ACM.

\bibitem[Kern et~al., 2004]{KernEtAl2004}
Kern, S., M{\"u}ller, S.~D., Hansen, N., B{\"u}che, D., Ocenasek, J., and
  Koumoutsakos, P. (2004).
\newblock Learning probability distributions in continuous evolutionary
  algorithms -- a comparative review.
\newblock {\em Natural Computing}, 3(1):77--112.

\bibitem[Kotthoff, 2016]{Kotthoff2016}
Kotthoff, L. (2016).
\newblock {\em Algorithm Selection for Combinatorial Search Problems: A
  Survey}, pages 149--190.
\newblock Springer.

\bibitem[Lehre and \"{O}zcan, 2013]{LehreOzcan2013}
Lehre, P.~K. and \"{O}zcan, E. (2013).
\newblock A runtime analysis of simple hyper-heuristics: To mix or not to mix
  operators.
\newblock In {\em Proceedings of the Workshop on Foundations of Genetic
  Algorithms}, FOGA `13, pages 97--104. ACM.

\bibitem[Lehre and Witt, 2012]{LehreWitt2012}
Lehre, P.~K. and Witt, C. (2012).
\newblock Black-box search by unbiased variation.
\newblock {\em Algorithmica}, 64(4):623--642.

\bibitem[Lissovoi et~al., 2017]{LissovoiEtAl2017}
Lissovoi, A., Oliveto, P.~S., and Warwicker, J.~A. (2017).
\newblock On the runtime analysis of generalised selection hyper-heuristics for
  pseudo-boolean optimisation.
\newblock In {\em Proceedings of the Genetic and Evolutionary Computation
  Conference}, GECCO '17, pages 849--856. ACM.

\bibitem[Lissovoi et~al., 2019]{LissovoiEtAl2019}
Lissovoi, A., Oliveto, P.~S., and Warwicker, J.~A. (2019).
\newblock On the time complexity of algorithm selection hyper-heuristics for
  multimodal optimisation.
\newblock In {\em AAAI Conference on Artificial Intelligence}, AAAI `19.
\newblock To Appear.

\bibitem[L{\'o}pez-Camacho et~al., 2014]{LopezEtAl2014}
L{\'o}pez-Camacho, E., Terashima-Marin, H., Ross, P., and Ochoa, G. (2014).
\newblock A unified hyper-heuristic framework for solving bin packing problems.
\newblock {\em Expert Systems with Applications}, 41(15):6876--6889.

\bibitem[Maden et~al., 2009]{MadenEtAl2009}
Maden, {\.I}., Uyar, S., and {\" O}zcan, E. (2009).
\newblock Landscape analysis of simple perturbative hyperheuristics.
\newblock In {\em Conference on Soft Computing}, Mendel `09, pages 16--22.
  Mendel.

\bibitem[Moraglio and Sudholt, 2017]{MoraglioSudholt2017}
Moraglio, A. and Sudholt, D. (2017).
\newblock Principled design and runtime analysis of abstract convex
  evolutionary search.
\newblock {\em Evolutionary Computation}, 25(2):205--236.

\bibitem[Nareyek, 2004]{Nareyek2004}
Nareyek, A. (2004).
\newblock Choosing search heuristics by non-stationary reinforcement learning.
\newblock In {\em Metaheuristics: Computer Decision-Making}, pages 523--544.
  Springer.

\bibitem[Ochoa et~al., 2009a]{OchoaEtAl2009GECCO}
Ochoa, G., Qu, R., and Burke, E.~K. (2009a).
\newblock Analyzing the landscape of a graph based hyper-heuristic for
  timetabling problems.
\newblock In {\em Proceedings of the Genetic and Evolutionary Computation
  Conference}, GECCO '09, pages 341--348. ACM.

\bibitem[Ochoa et~al., 2009b]{OchoaEtAl2009CEC}
Ochoa, G., V{\'a}zquez-Rodr{\'\i}guez, J.~A., Petrovic, S., and Burke, E.
  (2009b).
\newblock Dispatching rules for production scheduling: a hyper-heuristic
  landscape analysis.
\newblock In {\em IEEE Congress on Evolutionary Computation}, CEC `09, pages
  1873--1880. IEEE.

\bibitem[Oliveto and Witt, 2014]{OlivetoWitt2014}
Oliveto, P.~S. and Witt, C. (2014).
\newblock On the runtime analysis of the simple genetic algorithm.
\newblock {\em Theoretical Computer Science}, 545:2 -- 19.

\bibitem[Oliveto and Witt, 2015]{OlivetoWitt2015}
Oliveto, P.~S. and Witt, C. (2015).
\newblock Improved time complexity analysis of the simple genetic algorithm.
\newblock {\em Theoretical Computer Science}, 605:21 -- 41.

\bibitem[Oliveto and Yao, 2011]{OlivetoYao2011}
Oliveto, P.~S. and Yao, X. (2011).
\newblock Runtime analysis of evolutionary algorithms for discrete
  optimization.
\newblock In Auger, A. and Doerr, B., editors, {\em Theory of Randomized Search
  Heuristics: Foundations and Recent Developments}, pages 21--52. World
  Scientific.

\bibitem[{\"{O}}zcan et~al., 2010]{OzcanEtAl2012}
{\"{O}}zcan, E., Misir, M., Ochoa, G., and Burke, E.~K. (2010).
\newblock A reinforcement learning - great-deluge hyper-heuristic for
  examination timetabling.
\newblock {\em Applied Metaheuristic Computing}, 1(1):39--59.

\bibitem[Wald, 1944]{Wald1944}
Wald, A. (1944).
\newblock On cumulative sums of random variables.
\newblock {\em The Annals of Mathematical Statistics}, 15(3):283--296.

\bibitem[Witt, 2006]{Witt2006}
Witt, C. (2006).
\newblock Runtime analysis of the ($\mu$ + 1) {EA} on simple pseudo-boolean
  functions.
\newblock {\em Evolutionary Computation}, 14(1):65--86.

\bibitem[Wolpert and Macready, 1997]{WolpertMacready1997}
Wolpert, D.~H. and Macready, W.~G. (1997).
\newblock No free lunch theorems for optimization.
\newblock {\em IEEE Transactions on Evolutionary Computation}, 1(1):67--82.

\bibitem[Xu et~al., 2008]{XuEtAl2008}
Xu, L., Hutter, F., Hoos, H.~H., and Leyton-Brown, K. (2008).
\newblock Satzilla: Portfolio-based algorithm selection for sat.
\newblock {\em Journal of Artificial Intelligence Research}, 32(1):565--606.

\end{thebibliography}

\end{document}